\documentclass[twoside]{article}

\usepackage[accepted]{aistats2024}
\usepackage[utf8]{inputenc} 
\usepackage[T1]{fontenc}    
\usepackage{comment}
\usepackage{hyperref}       
\usepackage{url}            
\usepackage{nicefrac}       
\usepackage{microtype}     
\newcommand{\algemph}[3]{\algcolor{#1}{#2}{#3}}
\usepackage{xcolor}
\usepackage{framed}
\colorlet{shadecolor}{pink} 
\usepackage{authblk}
\usepackage{adjustbox}
\usepackage{comment} 
\usepackage{graphicx}
\usepackage{soul}
\usepackage{subcaption}
\usepackage{booktabs}
\usepackage{tablefootnote}
\usepackage{multirow}

\usepackage{natbib}
\usepackage{amsmath,amsthm,amssymb,amsfonts}
\usepackage{algorithm}
\usepackage{algorithmic}
\usepackage{enumerate}
\usepackage{cleveref}
\usepackage{comment}
\usepackage{bm}
\usepackage{pifont}
\usepackage{xcolor}

\usepackage[colorinlistoftodos,bordercolor=orange,backgroundcolor=orange!20,linecolor=orange,textsize=scriptsize]{todonotes}

\theoremstyle{plain} 
\newcommand{\vertiii}[1]{{\left\vert\kern-0.25ex\left\vert\kern-0.25ex\left\vert #1 
\right\vert\kern-0.25ex\right\vert\kern-0.25ex\right\vert}}
  
\usepackage{epsf}
\usepackage{graphics}
\usepackage{subcaption}
\usepackage{wrapfig}
\usepackage{psfrag}

\usepackage{color}

\usepackage{mathtools}
\usepackage{amsfonts}
\usepackage{amsthm}
\usepackage{amsmath}
\usepackage{amssymb}

\newtheorem{theorem}{Theorem}
\newtheorem{proposition}{Proposition}
\newtheorem{lemma}{Lemma}

\newtheorem{definition}{Definition}
\newtheorem{remark}{Remark}

\newtheorem{assumption}{Assumption}

\long\def\comment#1{}

\newcommand{\frakR}{\mathfrak{R}}

\newcommand{\norm}[1]{\left\| #1 \right\|}

\newcommand{\inprod}[2]{\ensuremath{\left\langle #1 , \, #2 \right\rangle}}
\newcommand{\pare}[1]{\ensuremath{\left( #1  \right)}}

\newcommand{\cbr}[1]{\left\{ #1 \right\}}

\newcommand{\unif}{\mathrm{unif}}
\newcommand{\SA}{\mathrm{SA}}

\newcommand{\abs}[1]{\left|#1\right|}

\newcommand{\E}{\ensuremath{{\mathbb{E}}}}

\newcommand{\Prob}{\ensuremath{{\mathbb{P}}}}

\DeclareMathOperator{\trace}{\mathsf{tr}}

\newcommand{\cE}{\mathcal{E}}

\newcommand{\R}{\mathbb{R}}

\newcommand{\cK}{\mathcal{K}}
\newcommand{\cC}{\mathcal{C}}  
\newcommand{\cD}{\mathcal{D}}  
\newcommand{\cF}{\mathcal{F}}  
\newcommand{\cH}{\mathcal{H}}  
\newcommand{\cG}{\mathcal{G}}  
\newcommand{\cQ}{\mathcal{Q}}
\newcommand{\sP}{\mathcal{P}}

\newcommand{\cR}{\mathcal{R}}  
\newcommand{\cN}{\mathcal{N}}   
\newcommand{\cS}{\mathcal{S}}
\newcommand{\cT}{\mathcal{T}}   
\newcommand{\cW}{\mathcal{W}}   
\newcommand{\cV}{\mathcal{V}}  
\newcommand{\cL}{\mathcal{L}} 
\newcommand{\cU}{\mathcal{U}}   
\newcommand{\cX}{\mathcal{X}}  

\DeclareMathOperator*{\argmin}{arg\,min}

\newcommand{\btheta}{\boldsymbol{\theta}}
\newcommand{\bsigma}{\boldsymbol{\sigma}}

\newcommand{\ba}{\mathbf{a}}

\newcommand{\bx}{\mathbf{x}}

\newcommand{\bv}{\mathbf{v}}
\newcommand{\bg}{\mathbf{g}}

\newcommand{\bw}{\mathbf{w}}
\newcommand{\bz}{\mathbf{z}}

\newcommand{\by}{\mathbf{y}}

\newcommand{\mA}{\mathbf{A}}
\newcommand{\mW}{\mathbf{W}}
\newcommand{\mS}{\mathbf{S}}

\newcommand{\mB}{{\bf B}}

\newcommand{\mF}{{\bf F}}

\newcommand{\mH}{{\bf H}}

\newcommand{\mV}{{\bf V}} 
\newcommand{\mX}{{\bf X}}

\newcommand{\mZ}{{\bf Z}}

\newcommand{\iid}{{\em i.i.d.~}}

\definecolor{MITBrown}{RGB}{164, 31, 50}

\usepackage{algorithm}
\usepackage[ruled,algo2e]{algorithm2e}
\usepackage{hyperref}
\hypersetup{colorlinks=true,linkcolor=blue,filecolor=magenta,urlcolor=cyan,}

\definecolor{antiquewhite}{rgb}{0.98, 0.92, 0.84} 
\definecolor{blizzardblue}{rgb}{0.67, 0.9, 0.93}
\newcommand{\algcolor}[3]{\hspace*{-\fboxsep}\colorbox{#1}{\parbox{#2\linewidth}{#3}}}
\begin{document}

%

%

\twocolumn[

\aistatstitle{On the Generalization Ability of Unsupervised Pretraining}

\aistatsauthor{Yuyang Deng \And Junyuan Hong \And  Jiayu Zhou \And Mehrdad Mahdavi }

\aistatsaddress{   Penn State University\And  Michigan state university \And  Michigan state university \And  Penn State University} ]
\begin{abstract}
Recent advances in unsupervised learning have shown that unsupervised pre-training, followed by fine-tuning, can improve model generalization. However, a rigorous understanding of how the representation function learned on an unlabeled dataset affects the generalization of the fine-tuned model is lacking. Existing theoretical research does not adequately  account for the heterogeneity of the distribution and tasks in pre-training and fine-tuning stage.
To bridge this  gap, this paper introduces a novel theoretical framework that illuminates the critical factor influencing the transferability of knowledge acquired during unsupervised pre-training to the subsequent fine-tuning phase, ultimately affecting the generalization capabilities of the fine-tuned model on downstream tasks. We apply our theoretical framework to analyze generalization bound of two distinct scenarios: Context Encoder pre-training with deep neural networks and Masked Autoencoder pre-training with deep transformers, followed by fine-tuning on a binary classification task.   Finally, inspired by our findings, we propose a novel regularization method during pre-training to further enhances the generalization of fine-tuned model.  
Overall, our results contribute to a better understanding of unsupervised pre-training and fine-tuning paradigm, and can shed light on the design of more effective pre-training algorithms.
\end{abstract}

\section{Introduction}

Unsupervised representation learning has achieved remarkable success in various domains, including computer vision and natural language processing, as evidenced by a rapidly increasing number of empirical studies~\citep{coates2012learning,radford2015unsupervised,sun2019videobert,dosovitskiy2020image,feichtenhofer2022masked,he2020momentum,he2022masked,devlin2018bert,chen2020simple}. In this learning paradigm, the goal is to learn a representation function on a large, possibly unlabeled dataset by optimizing a carefully designed unsupervised learning objective. Then, using the learned representation, a task-specific classifier, such as the head of a neural network, is trained on a small in-house dataset during the fine-tuning stage. This two-stage   paradigm addresses the issue of small dataset size in downstream tasks. While unsupervised pre-training for transfer learning has experienced significant empirical growth, a comprehensive understanding of the fundamental factors that influence the generalization performance of fine-tuned models lags considerably behind what has been empirically observed~\citep{neyshabur2020being}. 

Most existing generalization bounds primarily rely on notions such as  distance between the weights of the pre-trained and fine-tuned models~\citep{li2021improved,shachaf2021theoretical}  or data-dependent measurements such as Hessian~\citep{ju2022robust} through   PAC-Bayesian analysis~\citep{arora2018stronger,neyshabur2018pac} to examine the performance of fine-tuned model. These results inform the design of effective regularization methods~\citep{li2021improved,ju2022robust} or incorporating consistent losses~\citep{ju2022robust} in fine-tuning stage to improve the generalization of fine-tuned model by mitigating issues such as overfitting  caused by fine-tuning a large model on a small training set or instability due to label noise. These generalization bounds, however,  do not explicitly incorporate other key factors that may govern the success of fine-tuning  such as similarity between the pre-training  (on which a model is pre-trained) and
target tasks~\citep{shachaf2021theoretical} or task diversity~\citep{tripuraneni2020theory}, the number of training samples and complexity of model spaces utilized in each stage in a \textit{unified bound}.  For example,  in real-world learning tasks, the pre-training and fine-tuning tasks may be conducted on completely different domains, and we usually employ  some kind of transformation on the pre-training data (i.e., adding noise, rotating or masking), which further exacerbates the data heterogeneity. Consequently, a well-designed generalization theory is expected to take the data heterogeneity into account~\citep{yang2020analysis}. In modern transfer learning, different tasks can be conducted in the pre-training and fine-tuning stages. For example, in a Masked Autoencoder (MAE)~\citep{he2022masked}, a regression task utilized during  pre-training, while a classification task used for fine-tuning. Therefore, a desired theory should allow for flexibility in choosing diverse types of tasks in the pre-training and fine-tuning stages which poses a challenge in formalizing the desired guarantees.


Motivated by the above observations, we aim at formalizing and establishing  general generalization bounds on unsupervised pre-training and fine-tuning paradigm that captures aforementioned factors in a unified manner.
We introduce the  notion of {\em representation transferrability} to quantify how much knowledge can be transferred from unsupervised representation learning stage to fine-tuned model, in the presence of task heterogeneity. We then  establish a bound on the generalization capability of fine-tuned model composed with pre-trained representation model that highlights how   representation-induced complexity  and distribution mismatch affects the generalization of fine-tuned model. We instantiate our theory to  the scenario of Context Encoder~\citep{pathak2016context} with deep neural networks and
Masked Autoencoder~\citep{devlin2018bert,he2022masked} with deep Transformer architectures  which highlights the relative merits of  learning representations. From a technical perspective, we establish  generalization bounds for multi-layer transformers, by deriving the worst case covering number of hypothesis space   by expanding upon the machinery that was developed in~\citep{edelman2022inductive}.  


Since our theoretical analysis reveals  the representation-induced Rademacher complexity as one of the  key facotors governing the  capacity of the transfer learning, it naturally motivates itself to be incorporated as a regularizer during pre-training. Inspired by this observation, we propose a novel  {\em Rademacher representation regularized} algorithm, dubbed as {\sffamily{RadReg}}, to enhance the generalization capability of the fine-tuned model. We show that by utilizing unlabeled data from the downstream task, we can effectively regularize the pre-trained model to learn representations that entail better generalization after fine-tuning.   We propose an efficient algorithm to optimize the new objective and  establish its convergence on smooth nonconvex losses.

\vspace{-1mm}
\noindent\textbf{Contributions.} Our main
contributions are summarized as follows: \vspace{-1mm}
\begin{itemize}\vspace{-1.5mm}
    \item (Theory) We introduce a formal framework to study the utility of unsupervised representation learning and fine-tuning paradigm (Section~\ref{sec:formal}) and derive the  generalization bound for fine-tuned model based on a pre-trained representation function (Section~\ref{sec:main}). We discover that the generalization capability of  model depends on four key factors: Representation transferrability, representation-induced Rademacher complexity,  domain heterogeneity, and  generalization of the pre-training task. \vspace{-1.5mm} 
    \item (Applications) We apply our theory to derive generalization bound of the pre-training with a context encoder (CE) and a masked autoencoder (MAE) with a transformer   followed by a binary classification fine-tuning task (Section~\ref{sec:utility}). We show that, the pre-training tasks defined by regression loss are provably transferrable to downstream binary classification task. 
    In doing so, to our best knowledge, we establish the first generalization analysis of multi-layer transformer models with residual block. 
    \vspace{-1.5mm} 
      \item (Algorithm) Inspired by our generalization bounds, we propose a novel Rademacher Representation Regularized  algorithm,  {\sffamily{RadReg}}, for improved pre-training and provide  convergence guarantees for nonconvex objectives (Section~\ref{sec:radreg}). The experimental results show that {\sffamily{RadReg}} can learn better representation than $\ell_2$ norm regularized training on downstream tasks with a small dataset (Section~\ref{sec short exp}). 
\end{itemize}

\section{Additional Related Works} \label{sec:related}
\textbf{Theory of Transfer Learning} A significant body of work~\citep {tripuraneni2020theory,du2020few} focuses on the theoretical aspects of transfer learning paradigm, trying to answer the question: {\em why transfer learning can work and what factors affect the learning performance?}~\cite{tripuraneni2020theory} give the first risk bound capturing  {\em task diversity} among pre-training and fine-tuning stages as the key quantity affecting generalization which is also reflected in our generalization analysis.~\cite{xu2021representation} follow the setup in~\citep{tripuraneni2020theory}, and show that even though the model architectures used in representation learning and fine-tuning are different, the task diversity remains bounded. ~\cite{du2020few} study few-shot representation learning, where it considers pre-training a linear representation function by solving the regression  problem with squared loss (OLS) on a give large dataset, and then fine-tuning another linear predictor on some target dataset. It  shows that the generalization will depend on the number of pre-training data and fine-tuning data. Similar dependencies appear in the generalization bound obtained in our main theorem. \cite{zhang2023trade} study the general supervised pretraining, and highlight the trade-off between the intra and inter class diversity.

\textbf{Theory of Modern Unsupervised Representation Learning.}~
Recently, due to the rise of contrastive learning~\citep{chen2020simple} and masked training~\citep{devlin2018bert,he2022masked}, a line of studies are devoted to understanding the generalization capability or smaple complexity of  these learning paradigms~\citep{haochen2021provable,arora2019theoretical,wang2020understanding,lee2021predicting,ge2023provable,gouk2020distance,ju2022robust}.~\cite{arora2019theoretical}  presents a theoretical framework for studying contrastive learning, and shows that it provably reduces the sample complexity of downstream tasks. \cite{haochen2021provable} consider contrastive learning and establishes the theory without conditional independence of positive data pairs. \cite{wang2020understanding} prove that contrastive learning optimizes for alignment and uniformity asymptotically. \cite{zhang2022mask}  establish the connection of  masked pre-training with contrastive learning over bipartite graphs.~\cite{lee2021predicting} also consider a masking pre-training scenario, but contrary to the present work, it assumes the labels are generated by a function of masked data plus Gaussian noise, and only focuses on the ERM model as a representation function. A recent work~\citep{ge2023provable} also examine the unsupervised pre-training framework, but the difference to ours, they consider a maximum likelihood estimation as pre-training method, while we start from general pre-training task and instantiate it in modern machine learning scenario such as Context Encoder and MAE. \cite{gouk2020distance} study the end-to-end finetuing scenario, and find that the generalization of finetuned model will depend on the distance that neural network weights traveled away from pretrained model. They hence propose a distance regularization finetuning algorithm and achieve better performance. \cite{ju2022robust} also study the entire model finetuning paradigm, and derive a Hessian based generalization bound via PAC-Bayesian analysis. 

 \section{A Formal Framework}\label{sec:formal}
In this section we  formalize unsupervised pre-training followed by supervised fine-tuning problem that will enable us  to study the relative merits of various unsupervised representation learning approaches and  examine their utility on the generalization capability of  downstream tasks. In the scenario of unsupervised representation pre-training and   fine-tuning on a downstream task, we are given two datasets: one, possibly large,  unlabeled pre-training dataset and a small labeled data set.  The goal is to learn  model $f\circ h$ which is composed of task-specific function $f \in \cF$ and representation function $h \in \cH$, where $\cF$ and $\cH$ are model spaces for fine-tuned and representation models, respectively.

\noindent\textbf{Unsupervised pre-training.} We assume access to a  raw pre-training data  $\{\tilde{\bx}_i \}_{i=1}^N$  drawn from an unknown, arbitrary distribution $\cD$  over an instance domain $\mathcal{X}$ such as images. To learn representations, one first transforms (e.g., masking, adding noise, rotating, or other geometric transformation) unlabeled  data into $\tilde \bz_i = T_1(\tilde \bx_i) \in \mathcal{X}$ and (self-generated) label  $\tilde \by_i = T_2(\tilde \bx_i) \in \mathcal{Z}$ using suitable  transformers $T_1: \mathcal{X} \mapsto \mathcal{X}$ and $T_2: \mathcal{X} \mapsto \mathcal{Z}$ to generate the  pre-training dataset $\widehat{\cU} = \{(\tilde \bz_i, \tilde \by_i)\}_{i=1}^{N}$. For example, in  pre-training with masking, our augmented data are masked sentence/image, and self-generated labels are the masked part of data. We denote the transformed distribution over $\mathcal{X} \times \mathcal{Z}$ as $\cU$. We note that that marginal distribution $\cU_{\cX}$ of $\cU$ over instance space $\cX$ is not necessarily same as $\cD$ of raw data due to randomness in data transformation $T_1$. To learn the representations, we consider a class of decoding and encoding pairs, which is closely inspired by~\citep{hazan2016non}, and minimize the following empirical risk 
\begin{align}
   \min_{g\in\cG,h\in\cH}   \cL_{\widehat{\cU} }(g\circ h ) := \frac{1}{N} \sum\nolimits_{(\tilde\bz_i, \tilde\by_i) \in \widehat{\cU} } \ell(g\circ h (\tilde{\bz}_i),\tilde{\by}_i), \label{eq:pretrain obj}
\end{align}
where { $\cG \subseteq \{\mathcal{I} \mapsto \mathcal{Z}\}$ and $\cH\subseteq \{\mathcal{X} \mapsto \mathcal{I}\}$} are the model spaces for encoder and decoder, respectively, where $\mathcal{I}$ denotes the latent space of representations, and $\ell$ is the loss function used for  pre-training, e.g., $\ell(g\circ h (\tilde{\bz}_i),\tilde{\by}_i) = \|g\circ h (\tilde{\bz}_i)-\tilde{\by}_i\|_2^2$.  Let $\hat{g}$ and $\hat{h}$ denote the decoder and encoder (representation function) obtained by solving~(\ref{eq:pretrain obj}). We define the following excess risk for pre-training task: 
\begin{align*}
    \cE_{\cU}(\hat{g}, \hat{h}) := \cL_{\cU}(\hat{g}\circ \hat{h}) - \min_{g\in\cG, h \in \cH} \cL_{\cU}(g\circ h )
\end{align*}
where $\cL_{\cU}(g\circ h) := \E_{(\tilde\bz,\tilde\by)\sim \cU}[\ell(f\circ h(\tilde\bz),\tilde\by)]$ denotes the generalization ability of pre-training task realized by distribution $\mathcal{U}$. We note that  the learned decoder function $\hat{g}$ may be discarded after pre-training. We use \textcolor{black}{$h^*_{\cU} = \arg\min_{h\in\cH} \min_{g\in\cG} \cL_{\cU}(g\circ h) \in \cH$} to denote optimal encoder for pre-training task.

\noindent\textbf{Supervised fine-tuning.} In fine-tuning stage, we assume access to a labeled downstream dataset $\widehat{\cT} = \{\bx_i, \by_i\}_{i=1}^{n}$ where feature vector $\bx_i$ is sampled based an unknown, arbitrary distribution $\cT$ (possibly different from $\cD$) on domain $\cX$, and  its label $\by_i$ is generated based on a labeling function $\by_i = y(\bx_i)$. The goal is to utilize the representation function  $\hat{h}$ obtained by solving~(\ref{eq:pretrain obj}) to  perform \textit{fine-tuning}  on the downstream dataset $\widehat{\cT}$ to learn a prediction model $\hat{f}$ from a function class $\cF$:
\begin{align}
    \min_{f\in\cF} \cR_{\widehat{\cT}}(f\circ \hat{h} ) := \frac{1}{n}\sum\nolimits_{(\bx_i, \by_i)\in \widehat{\cT}} \phi(f\circ \hat{h} (\bx_i),\by_i), \label{eq:fine-tune obj}
\end{align}
where $\phi$ is the loss function  which is not necessarily the same as the  pre-training loss.

Our goal is to rigorously analyze the generalization capability of the  final model which is the composition of two functions, i.e.,  $\hat{f}\circ \hat{h}$ where $\hat{f}$ is the solution of (\ref{eq:fine-tune obj}), by bounding the excess risk 
\begin{equation}
\mathcal{E}_{\cT}(\hat{f}, \hat{h}) = \cR_{ \cT}(\hat{f}\circ \hat{h}) -\min_{f\in \cF,h\in\cH} \cR_{ \cT}(f\circ h).   
\end{equation} 
Here
$\cR_{\mathcal{T}}(f\circ h) := \E_{\bx\sim \cT}[\phi(f\circ h(\bx),y(\bx))]$ denotes the  true risk on downstream task realized by distribution $\mathcal{T}$ over $\cX$ and underlying labeling function $y(\cdot)$.

 \section{On the Utility of Unsupervised Representation Learning}\label{sec:main}
We now turn to establishing the generalization bound of fine-tuned models given a pre-trained representation function, and discuss its implications. Before, we first introduce two key notions. We start by introducing the notion of Rademacher complexity of a hypothesis space when individual models are composed with a fixed representation function (similar  measures  appear in~\citep{tripuraneni2020theory,xu2021representation}).

 \begin{definition} [{\sffamily{Representation-induced Rademacher complexity}}] 
    For a hypothesis space $\cF$ of set of real (vector)-valued  functions defined over input space $\mathcal{X}$ and label space $\mathcal{Y}$,  a loss function $\phi: \mathcal{Y}\times \mathcal{Y} \mapsto \mathbb{R}_{+}$, 
    and a dataset  $\widehat{\cT} = \{\bx_i, \by_i\}_{i=1}^{n}$,  the empirical Representation-induced Rademacher complexity of $\cF$ with respect to $\phi$ and $\widehat{\cT}$,    {for a  given representation function $\hat{h}$}, is defined as  
    \begin{equation*} 
    \begin{aligned}
           & {\frak{R}}_{\widehat \cT}(\phi\circ \mathcal{F}\circ \hat{h}) \\ &:= \mathbb{E}_{\boldsymbol \varepsilon \in \{\pm1\}^n} \left[\sup_{f \in \mathcal{F}} \frac{1}{n}\sum\nolimits_{i=1}^{n} \varepsilon_i   \phi(f\circ \hat{h} (  \mathbf{x}_i ),\by_i) \right], 
        \end{aligned}
    \end{equation*} 
    where $\varepsilon_1, \ldots, \varepsilon_n$ are \iid Rademacher random variables with $\Prob \{ \varepsilon_i = 1 \} = \Prob \{ \varepsilon_i = -1 \} = {1}/{2}$.
\end{definition}
The following definition, relates the generalization of  fine-tuned and representation models. 
\begin{definition}[{\sffamily{Represnetation transferability}}]\label{def:representation:Transferability}
Given two representation functions $h,h'\in \cH$ and a distribution $\cU$ for pre-training data, we say a pre-training task and fine-tuning task satisfies  $(C_{\beta},\beta)$   transferability for some constant $0<C_{\beta}<\infty, \beta>0$ on $h,h'$, if  the following statement holds:
\begin{align*}
     &\min_{f\in \cF} \cR_{\cT}(f\circ h)   - \min_{f'\in \cF }\cR_{\cT}(f'\circ h')  \\
    \leq & C_{\beta} \pare{ \min_{g\in\mathcal{G}}\cL_{\cU}(g\circ h)   - \min_{g'\in\mathcal{G}}\cL_{\cU}(g'\circ h') }^{\beta}
    \end{align*} 
    where $ \cR_{\cU_{\cX}}(f\circ h) := \E_{\bx\sim\cU_{\cX}}[\phi(f\circ h(\bx),y(\bx))]$ denotes the  risk realized by pre-training marginal data distribution $\cU_{\cX}$ and downstream labeling function $y(\cdot)$.
\end{definition}
We note that a similar notation is proposed in the analysis of multi-task learning~\cite[Definition~4]{hanneke2022no}, to characterize the transferrability from one task to another task.  We emphasize that {represnetation transferability} is the key to transfer the generalizability of pre-training model to fine-tuned model. Unlike the transfer ratio defined in previous works~\citep{tripuraneni2020theory,ge2023provable,zhang2023trade}, we have an exponent variable $\beta$, which allows the transferrability from losses with different order, e.g., from a quadratic loss to non-quadratic loss.  Later on, we show that condition holds essentially under realistic assumptions on suitable data transformations to generate pre-training data  and model spaces,   such as pre-training with a inpainting autoencoder and a masked autoencoder with a transformer, where both are  fine-tuned on a classification task.  

The next theorem establishes the generalization bound of the fine-tuned model on a downstream dataset, given a pre-trained representation function $\hat{h}$.
 \begin{theorem}\label{thm:main generalization}
 Assume $\hat{h}$ and $\hat{g}$ are the pre-trained representation function and its associated decoder function, and real valued non-negative loss $\phi$ to be $G_\phi$ Lipschitz and bounded by $B_\phi$. Assume pre-training and fine-tuning task admit $(C_\beta, \beta)$ representation transferrability on $\hat{h}$ and $h^*_\cU$.  If we solve (\ref{eq:fine-tune obj}) to get $\hat{f}$, then with probability at least $1-\nu$, the following statement holds  
\begin{align*}
   &\mathcal{E}_{\cT}(\hat{f}, \hat{h})    \leq C_\beta  \cE_{\cU}(\hat{g}, \hat{h})  ^\beta +4G_\phi {\frak{R}}_{\widehat \cT}( \cF\circ \hat{h})   \\
   &   + 4B_\phi\sqrt{\frac{\log(1/\nu)}{n}}  + 4B_\phi \norm{\cT-\cU_{\cX}}_{\mathrm{TV}} +   \min_{f\in\cF}\cE_{\cT}(f,  h^*_{\cU})
\end{align*}
where $h^*_{\cU} = \arg\min_{h \in \cH}  \min_{g\in\mathcal{G}} \cL_{\cU}(g\circ h)$ is the optimal pre-training representation function, and $\norm{\sP-\cQ}_{\mathrm{TV}} = \sup_{A \in \Omega} |\sP(A)-\cQ(A)|$ denotes total variation distance between two distributions.
\end{theorem}

The proof of Theorem~\ref{thm:main generalization} is deferred to Appendix~\ref{app:proof main generalization}. Theorem~\ref{thm:main generalization} shows that the generalization of the fine-tuned model depends on four quantities: i) Representation transferrability, ii) Representation-induced Rademacher complexity, iii) domain heterogeneity and iv) generalization of the pre-training task. 

Representation transferrability is the key to connect downstream generaliztion with pre-training generalization. It is analogous to task diversity notion in the multi-task learning works~\citep{tripuraneni2020theory,xu2021representation}, since they all measure how well the knowledge can be transferred across different learning stage. However, our notion is more powerful since we neither assume the pre-training and fine-tuning stage share the same type of task, e.g., both being regression, nor assume a generic nonlinear feature representation  is shared across all tasks~\citep{tripuraneni2020theory}. As we will see in the later section, with the help of representation transferrability, we can show that encoder learnt by regression pre-training can be transferred to downstream classification task. The representation-induced Rademacher complexity will play a key role in reflecting how well the learnt representation and $\cF$ are coupled. Notice that this complexity is defined over downstream data, and only over class $\cF$, which means that in fine-tuning stage we only suffer from a smaller complexity in learning. The price for  learning with potential more complex encoder class $\cH$ is paid in pre-training task. This observation is consistent with a line of multi-task or transfer learning works~\citep{tripuraneni2020theory,du2020few,xu2021representation,ge2023provable}.  

The domain heterogeneity term $\norm{\cT-\cU_{\cX}}_{\mathrm{TV}}$ characterizes the statistical heterogeneity between the pre-training task and the fine-tuning task.   Generalization of the pre-training task also appears in the bound which depends on the convergence of the optimization algorithm, and the complexity of the representation function  classes  $\mathcal{G}$ and $\mathcal{H}$ for decoder and encoder, respectively. The last term $\min_{f\in\cF}\cE_{\cT}(f,  h^*_{\cU})$ is the downstream risk evaluated with optimal representation model $h^*_{\cU}$, which characterize the task heterogeneity between pre-training and downstream stages. If the two tasks are well aligned, i.e., they share similar optimal representation, then this quantity is ignorable.

 \section{The Power of Unsupervised Representation Learning}\label{sec:utility}
We  now proceed to establish generalization bounds in two distinct settings by refining the generic result presented in the previous section.

\subsection{Pre-training with Context Encoder}\label{sec:dae}
\noindent \textbf{The setting.~}We start by applying  our theory to the setting where inpainting task is considered as pre-training task and binary classification as downstream task. This learning paradigm is also known as Context Encoder (CE)~\citep{pathak2016context}, where in pre-training stage, a deep neural network is trained by reconstructing  a random transformation of raw data (e.g, rotating, scaling, adding Gaussian noise or masking) of a given image: 
\begin{align}
  \min_{g \in \cG, h \in \cH} \cL_{\widehat \cU}(g\circ h) :=  \frac{1}{N} \sum_{i=1}^N \norm{g(h(\tilde\bz_i)) - \bz_i }^2\label{eq: DAE pt obj}
\end{align}
to learn $\hat{g}$ and $\hat{h}$.  Then, we discard the decoder $\hat{g}$, and use the rest layers as an encoder. A linear projection head is added on top of encoder in fine-tuning stage, on the downstream binary classification task with data $\bx_1,\ldots,\bx_n$ using the learnt encoder $\hat{h}$:
\begin{align}
   \min_{f\in\cF} \cR_{\widehat\cT}(f\circ \hat{h})=   \frac{1}{n}\sum_{i=1}^n \phi(f(\hat{h}(\bx_i)),y_i). \label{eq: DAE ft obj}
\end{align}
The encoder-decoder architecture is defined as follows:
\begin{align*}
 \text{\sffamily{encoder:}} \quad   & h(\bx) =  \sigma\pare{\mW_L \cdots \sigma \pare{  \mW_1 \bx} }  , \\
 \text{\sffamily{decoder:}} \quad   & g(h(\bx)) =  \mW_{L+1} h(\bx), 
\end{align*}
where $\mW_1 \in \R^{m\times d}$, $\mW_2,\ldots,\mW_{L} \in \R^{m\times m}$, and $\mW_{L+1} \in \R^{d\times m}$. In fine-tuning stage, we add a linear head on top of encoder function, i.e., $f(h(\bx)) = \btheta^\top h(\bx)$.   The hypothesis class for encoder is then defined as:
\begin{align*}
    \cH := \cbr{\begin{aligned}\bx \mapsto  \sigma\pare{\mW_L\cdots \sigma\pare{\mW_1\bx}} : \norm{\mW_l} \leq W(l),\\
 \norm{\mW_l}_{2,1} \leq B(l)     \end{aligned} }
\end{align*}
where $W(l)$ and $B(l)$ are upper bound on spectral and $(2,1)$ norms of weight matrices, respectively. 

The  decoder class is defined as:
\begin{align*}
    \cG := \cbr{\begin{aligned}\bx \mapsto \mW_{L+1}\bx: \norm{\mW_{L+1}}\leq W(L+1),\\
    \norm{\mW_{L+1}}_{2,1}\leq B(L+1)   \end{aligned}}.
\end{align*}
\noindent \textbf{Generalization bound.~} The following lemma establishes the representation transferrability of CE pre-training to binary classification task.  We need to make the following assumption 

\begin{assumption}[Realizability]\label{assumption: realizable}
There exists $g^* \in \cG$ and $h_\cU^* \in \cH$ such that $ \cL_{\cU}(g^*\circ h_\cU^*) = 0$.

\end{assumption}
\begin{remark}
    In Assumption~\ref{assumption: realizable} we assume that there exist optimal encoder and decoder that can perfectly realize pre-training task. This is reasonable if we consider overparameterized model, e.g., deep neural network. For example, in masked image reconstruction pre-training, at the most cases, the remaining part of image is enough for deep model to reconstruct the raw image~\citep{pathak2016context,he2022masked}. 
\end{remark}
\begin{lemma}\label{lem: DAE transfer}
Under Assumption~\ref{assumption: realizable}, CE pre-training admits an $\pare{\Omega\pare{1},\frac{1}{2}}$ representation transferrability to binary classification task.  
\end{lemma}
The proof of Lemma~\ref{lem: DAE transfer} is deferred to Appendix~\ref{app:proof DAE transfer}.  This lemma shows that generalization of a pre-training regression task can be effectively transferred to a downstream binary classification task. Here the transfer exponent is $\frac{1}{2}$, which implies that the generalization risk of downstream task will be roughly square root of pre-training generalization. To get the excess risk rate of downstream task, we need to derive the generalization risk of neural network regression. The existing works~\citep{cao2019generalization,cao2020generalization,arora2019fine} mainly focus on classification task where the loss function is Lipschitz, which is not the case in regression loss. Our technique is to generalize the seminal analysis in~\citep{srebro2010smoothness} for smooth losses and scalar valued hypothesis classes to a vector valued hypothesis class, i.e., neural network class in our case and borrow the standard neural network covering number result from~\citep{bartlett2017spectrally} to conclude the proof.
\begin{theorem}\label{thm:DAE main}
Assume $\hat{h}$ and $\hat{f}$ are the pre-trained representation function and its associated decoder function obtained by solving~\eqref{eq: DAE pt obj} and~\eqref{eq: DAE ft obj}. Let $\tilde\mZ = [\tilde\bz_{1};\ldots;\tilde\bz_{N}]$ and $\mX  = [\bx_{1};\ldots;\bx_{N}]$ be pre-training and downstream data, then under Assumption~\ref{assumption: realizable} with probability at least $1-\nu$, the following statement holds:
\begin{align*}
    &\cE_{\cT}(\hat{f},\hat{h})    \\
    &\leq   \tilde O\pare{  \frac{\sqrt{s_{L+1}   \norm{\mX}^2}}{n}+\sqrt{\frac{ {  \norm{\tilde\mZ}^2   } s_{L+1} \pare{\sum_{l=1}^{L+1}\rho_l }^3} {N}} }\\
    &\  + 4B_\phi\pare{\sqrt{\frac{\log(\frac{1}{\nu})}{n}} +  \norm{\cT - \cU_{\cX}}_{\mathrm{TV}}} + \min_{f\in\cF}\cE_{\cT}(f,  h^*_{\cU} ), 
\end{align*}
where $s_l = \prod_{l=1}^{L+1} W^2(l), \rho_l={B(l)}/{W(l)}$ .
\end{theorem}
The proof of Theorem~\ref{thm:DAE main} is deferred to Appendix~\ref{app:proof DAE main}. Here we achieve roughly $O(\frac{\cC(\cF)}{\sqrt{n}} + \frac{\cC(\cG\circ\cH)}{\sqrt{N}})$ bound for downstream task where $\cC(\cdot)$ denotes the complexity of the set. The cost of learning the complex heavyweight encoder is incurred during the pre-training task, whereas in the fine-tuning stage, we only endure the complexity of learning a lightweight classification head.


\subsection{Pre-training with masked autoencoder  with tranformer models}

\label{sec:mask}
\noindent \textbf{The setting.~}Here we apply our theory to explain the empirical success of masked autoencoder pre-training methods~\citep{he2022masked}.  In  masked autoencoder pre-training, taking vision tasks for example, we draw a large set of images $\mZ_1,...,\mZ_N \in \R^{K\times d}$, and then randomly mask some patches of each image to get $\tilde{\mZ}_1, \ldots ,\tilde{\mZ}_N\in \R^{K\times d}$. Then an encoder-decoder model is trained by recovering the missing  patches:
\begin{align}
  \min_{g \in \cG, h \in \cH} \cL_{\widehat \cU} (g\circ h) := \frac{1}{N}  \sum_{i=1}^N \norm{g(h(\tilde\mZ_i)) - \mZ_i }_{\text{F}}^2\label{eq: MAE pt obj}
\end{align}
to get $\hat{g}$ and $\hat{h}$.   Finally, we discard the decoder and only fine-tune a new head (e.g., linear projection layer) on the downstream binary classification task with data $\mX_1,\ldots,\mX_n$ using the  encoder:
\begin{align}
   \min_{f\in\cF} \cR_{\widehat\cT}(f\circ \hat{h})=   \frac{1}{n}\sum_{i=1}^n \phi(f(\hat{h}(\mX_i)),y_i). \label{eq: MAE ft obj}
\end{align}
We  consider an $L$-layer transformer as the pre-training encoder model, and a linear projection layer as the pre-train decoder model, and a linear projection layer for binary classification as fine-tune model.
\begin{align*}
 &\text{\sffamily{encoder:}} \quad   h(\mX) =  \SA_{\mW^L} \pare{ \SA_{\mW^{L-1}} \pare{\cdots \SA_{\mW^1} \pare{\mX }}},\\
 &\text{\sffamily{decoder:}} \quad    g(h(\mX)) =   \pare{h(\mX)}\mW_{D},
\end{align*}
where $\SA_{\mW}(\cdot)$ is a self-attention module parameterized by $\mW = \pare{\mW_V, \mW_Q, \mW_K, \mW_{\mathrm{FC1}}, \mW_{\mathrm{FC2}}  }$, which is defined as \begin{align*}
    \mathrm{SA}_{\mW}(\mX ) &= \alpha_2\sigma\pare{\mZ \mW_{\mathrm{FC1}}}\mW_{\mathrm{FC2}}+\mZ,\\
    \mZ &= \pare{\alpha_1\mA+\mX }, \\
    \mA &= \mathrm{softmax}\pare{\frac{1}{\sqrt{d_K}} \mX \mW_K (\mX\mW_Q)^\top } \mX \mW_V,
    \end{align*}
where $\alpha_1, \alpha_2$ are some small constant, as used in practice~\citep{noci2022signal}. We assume $l$th layer's weights' spectral norm is bounded by $W(l)$, and $(2,1)$ norm is bounded by $B(l)$.  In downstream task, we aggregate (sum) over all patches from encoder, add a linear projection head $\btheta$ on top of $h(\mX)$ to make a scalar output:
\begin{align*}
    &\text{\sffamily{downstream:}} \quad    f(h(\mX)) =   (\mathbf{1}^\top h(\mX)) \btheta.
\end{align*} 

\noindent \textbf{Generalization bound.~}The following lemma establishes the representation transferrability of MAE with a transformer pre-training to binary classification task.

\begin{lemma}\label{lem:MAE transfer}
MAE pre-training admits an $\pare{\Omega \pare{1},\frac{1}{2}}$ representation transferrability to binary classification task 
\end{lemma}
The proof of Lemma~\ref{lem:MAE transfer} is deferred to Appendix~\ref{app:proof MAE transfer}.
This implies that  MAE pre-training with  a multi-layer transformer can be transfered to binary classification task, with a constant factor. The exponent is also $1/2$. To derive the excess risk bound of downstream task, we need to find the generalization risk of transformer regression, which is characterized by the following lemma.
 \begin{lemma}[Generalization of MAE pre-training task]\label{lem:MAE gen}
Let $\hat{g},\hat{h}$ be the solution of~\eqref{eq: MAE pt obj}, and $\tilde\mZ_{[N]} = [\tilde\mZ_{1};\ldots;\tilde\mZ_{N}]$ is the concatenated pre-training data. Then under Assumption~\ref{assumption: realizable} with probability at least $1-\nu$ the following statement holds: 
\begin{align*}
     \cE_{\cU} (\hat{g}, \hat{h}) \leq O\pare{ s^2_{L}\norm{\tilde\mZ_{[N]}}^2 \sum_{l=1}^{L+1} \frac{\rho_l  }{N} + \frac{\log(\frac{1}{\nu})}{N}},
\end{align*} 
where  
\begin{align*} 
s_l &:= \prod_{j=1}^l \pare{\alpha_2  W^2(j) +1} \pare{ W^2(j)\alpha_1  K   + 1  },\\
     {\rho}_{l} &:= O\pare{ {(\alpha_1\alpha_2W^2(l) +  \alpha_1)^2  B^2(l) } \ln(2d^2)}\\
    &\quad \times \pare{K^2 + \frac{\alpha_1W^4(l)\pare{s_{l-1}\max_{i\in[N]} \norm{\mX_i}  }^4}{d_K}  }   \\
        &\quad+ O\pare{ {\alpha_2^2W^2(l) B^2(l) (1 + \alpha_1^2 K^2W^2(l) )}   \ln (2dm)} .
\end{align*} 
\end{lemma}
The proof of Lemma~\ref{lem:MAE gen} is deferred to Appendix~\ref{app:proof MAE generalization pre-train}. The proof idea is similar to analysis of CE pre-training, where we connect local Rademacher complexity to the covering number of model class, and then use techniques from~\citep{srebro2010smoothness} to establish the generalization of a smooth loss, i.e. MSE.   At the heart of our proof is to carefully control the norm of stacked output of transformer on $N$ samples, so that the final covering number does not scale with $N$, but only depends on spectral norm of concatenated samples. We note that  a similar study~\citep{edelman2022inductive} also establishes the capacity  of transformers, but they do not consider residual blocks.

We now proceed to determine the excess risk associated with fine-tuning a transformer model on a binary classification task.
\begin{theorem}\label{thm: MAE main}
 Assume $\hat{h}$ and $\hat{f}$ are the pre-trained representation function and its associated decoder function obtained by solving~\eqref{eq: MAE pt obj} and~\eqref{eq: MAE ft obj}. Let $\tilde\mZ_{[N]} = [\tilde\mZ_{1};...\tilde\mZ_{N}]$ and $\mX_{[N]} = [\mX_{1};...\mX_{N}]$ be pre-training and downstream data, then under Assumption~\ref{assumption: realizable}, with probability at least $1-\nu$, the following statement holds:
    \begin{align*}
        &\cE_{\cT} (\hat{f}, \hat{h})  \\
        & \leq O\pare{ \frac{\sqrt {s_L   \norm{\mX_{[N]}}^2}} {n}+ \sqrt{ \frac{s^2_{L}\norm{\tilde\mZ_{[N]}}^2  \sum_{l=1}^L\rho_l  }{N} } } \\
        &   + 4B_\phi\pare{\sqrt{\frac{\log(\frac{1}{\nu})}{n}} + \norm{\cT - \cU_{\cX}}_{\mathrm{TV}} } + \min_{f\in\cF}\cE_{\cT}(f,  h^*_{\cU} ),
    \end{align*}
where $s_l, \rho_l$ are constants as defined in Lemma~\ref{lem:MAE gen}.
\end{theorem}
The proof of Theorem~\ref{thm: MAE main} is deferred to Appendix~\ref{app:proof MAE main}.  Our observations are similar to those in the CE scenario: since we train the encoder only on the pre-training dataset, during downstream learning, we mainly contend with the intricacies of a smaller model class, which results in the generalization bound of  $O(\frac{\cC(\cF)}{\sqrt{n}} + \frac{\cC(\cG\circ\cH)}{\sqrt{N}})$. Meanwhile, it's worth noting that the introduction of masking may potentially reduce the norm of the data, denoted as $\norm{\tilde \mZ_{[N]}}^2$, thereby diminishing the influence of the second term. 
On the other hand, it could also amplify the domain discrepancy, i.e., $\norm{\cT - \cU_{\cX}}_{\mathrm{TV}}$.



\section{Effective  Learning via Rademacher Representation Regularization}\label{sec:radreg}
 As shown in Theorem~\ref{thm:main generalization}, a significant quantity that affects the generalization risk of downstream task is $\frak{R}_{\widehat{\cT}}(\phi\circ\cF\circ \hat{h} )$, the Rademacher complexity of $\cF$ given learnt representation function $\hat{h} $. Here we devise an algorithm to  leverage the unlabeled downstream data in the pre-training stage, to regularize the representation function and further improve the accuracy of fine-tuned model.
 
 Let us first consider binary classification case with binary label $y_i \in \{-1,+1\}$.  The idea is that, in the binary classification setting, the Rademacher complexity is {\em independent} of labels, and hence it can be precisely estimated by only unlabeled downstream dataset. If we assume $\phi$ is $G_\phi$ Lipschitz, then according to the contraction property of Rademacher complexity~\citep{ledoux2013probability}, we have:
 $ \frak{R}_{\widehat{\cT}}(\phi\circ\cF\circ \hat{h} ) \leq G_\phi \E_{\bsigma  } \left[\sup\nolimits_{f\in\cF}\frac{1}{n} \sum_{i=1}^n \sigma_i  f(\hat{h}(\bx_i))  \right]$,
 where we can see that, due to the randomness of Rademacher variables, the upper bound of Rademacher complexity can be estimated \emph{without} knowing the actual labels $\{y_i\}_{i=1}^n$. Hence, we can leverage the unlabeled downstream data in the pre-training stage, to regularize the representation function. This can be cast as the following problem:
 \begin{align*}
     \min_{g\in \mathcal{H},h \in \cH} \cL_{\widehat{\cU}} (g\circ h) +\lambda \E_{\bsigma  } \left[\sup\nolimits_{f\in\cF}\frac{1}{n} \sum_{i=1}^n  \sigma_i f(h(\bx_i))  \right],
 \end{align*}
 where $\lambda$ is the regularization coefficient. The  idea  can also be generalized to multi-class classification, where we use a vector contraction lemma to estimate the upper bound of this complexity which we discuss in Section~\ref{subsec:multi-class}. To estimate the expectation, we sample $B$ configurations of Rademacher variables $\{\bsigma^j = [\sigma^j_1,...,\sigma^j_n]\}_{j=1}^B$. If we assume $f$ and $h$ are parameterized by $\bv \in \cV$ and $\bw \in \cW$, respectively, we have
 \begin{align}\label{eq:radreg obj}
     \min_{\bw\in \cW} \cL_{\widehat{\cU}} ( \bw) +  \frac{\lambda}{B}\sum_{j=1}^B \left[\max_{\bv_j \in \cV}\frac{1}{n} \sum_{i=1}^n   \frak{R}_j (\bv_j  ,\bw ;   \bx_i )  \right]  
 \end{align}
where $\frak{R}_j (\bv  ,\bw ;   \bx_i )  : = \sigma_i^j f_{\bv}(h_{\bw} (\bx_i))$. 


\begin{algorithm2e}[t]
\renewcommand{\algorithmicrequire}{\textbf{Input:}}
    \DontPrintSemicolon
    \caption{{\sffamily{RadReg}}: Rademacher Regularized Pre-training}
    \label{algorithm: RPT_stoch}
    \textbf{Input:} Number of iterations $T$;  regularization parameter $\lambda$ \\
     Sample $B$ configurations of Rademacher variables $\{\bsigma^1,\ldots ,\bsigma^B\}$, \\
    Initialize $\bv_j^{0} = \mathbf{0}, ~\forall j\in [B]$ \\

    \For{$t = 0,\ldots,T-1$}{  
       
    Sample a batch of data from pre-training dataset $\{\tilde \bz_1^t,\ldots, \tilde \bz_{n'}^t\}$ \\
        Sample a batch of data from downstreaming dataset $\{\tilde \bx_1^t,..., \tilde \bx_{n'}^t\}$ \\
        \For{$j = 1,...,B$}{
           
           \algemph{antiquewhite}{0.85}{  $\bv^{t+1}_j = \bv^{t}_j + \gamma \frac{1}{n'}\sum_{i=1}^{n'}  \nabla_{\bv} \frak{R}_j(\bv^t_j,\bw^t; \tilde \bx_i^t)$. \\
           \quad  \quad \quad  \quad \# Dual variable update}\\
            
        }
        \algemph{blizzardblue}{0.9 }{ $\bw^{t+1}  = \bw^{t} - \eta  \frac{1}{n'}\sum_{i=1}^{n'}\nabla \cL_{\widehat\cU}(\bw^{t};\tilde{\bz}^t_i) \\
          -\eta \lambda \frac{1}{B}\sum_{j=1}^B \left[\frac{1}{n'}\sum_{i=1}^{n'} \nabla_{\bw}\frak{R}_j(\bv^{t}_j,\bw^{t}; \tilde \bx_i^t)\right]  $ \\ \  \#      Representation model update}
         
    }
\textbf{Output:} $\hat{\bw}$ uniformly  sampled from $\{\bw^t\}_{t=1}^T$. 
\end{algorithm2e} 
\noindent\textbf{Optimization method.~}To solve the aforementioned optimization problems,  we adapt the celebrated SGDA algorithm~\citep{lin2019gradient} (Algorithm~\ref{algorithm: RPT_stoch}). At the beginning of each iteration, we first sample a batch of $n'$ pre-training data $\{\bz_i^t \}_{i=1}^{n'}$, and then do mini-batch stochastic gradient descent: 
\begin{align*}
    \bw^{t+1} &= \bw^{t} - \eta  \frac{1}{n'}\sum\nolimits_{i=1}^{n'}\nabla \cL_{\widehat{\cU}}(\bw^{t};\tilde{\bz}^t_i)  \\
    &\quad -\eta  \lambda \frac{1}{B}\sum\nolimits_{j=1}^B \left[\frac{1}{n'}\sum\nolimits_{i=1}^{n'} \nabla_{\bw}\frak{R}_j(\bv^{t}_j,\bw^{t}; \tilde \bx_i^t)\right]. 
\end{align*}
To solving the inner max problem, we sample another batch of $n'$ downstream (unlabeled) data, and then we  do one step mini-batch stochastic gradient ascent:$
\bv^{t+1}_j = \bv^{t}_j + \gamma \frac{1}{n'}\sum\nolimits_{i=1}^{n'}  \nabla_{\bv} \frak{R}_j(\bv^t_j,\bw^t; \tilde \bx_i^t) .$
 
\noindent\textbf{Convergence analysis of RadReg.~}To establish the convergence of {\sffamily{RadReg}} on (\ref{eq:radreg obj}), we consider the following primal function:$
    \Psi(\bw) := \cL_{\widehat{\cU}} ( \bw) + \lambda \frac{1}{B}\sum\nolimits_{j=1}^B \left[\max_{\bv_j \in \cV}\frac{1}{n} \sum\nolimits_{i=1}^n   \frak{R}_j (\bv_j  ,\bw ;   \bx_i )  \right].$
Then we follow~\citep{lin2020gradient} and consider the following Moreau envelope function:
\begin{definition}[Moreau Envelope] A function $\Psi_{\rho} (\bw)$ is the $\rho$-Moreau envelope of a function $\Psi$ if $    \Psi_{\rho} (\bw) := \min_{\bw'\in \cW} \{ \Psi  (\bw') + \frac{1}{2\rho}\|\bw'-\bw\|^2\}$. 
 \end{definition}
 
 \begin{theorem}[Convergence of {\sffamily{RadReg}} with Linear Top Layer, Informal] \label{thm:convergence of radreg informal}
 {\sffamily{RadReg}} (Algorithm~\ref{algorithm: RPT_stoch} converge to $\epsilon$-stationary point of $\Psi_{1/4L} (\bw)$ with gradient complexity bounded by $  
     O\pare{{B}/{\epsilon^8}}.$ 
\end{theorem}
The formal version of Theorem~\ref{thm:convergence of radreg informal} as well as the proof is deferred to Appendix~\ref{app:proof convergence}. We can see that the proposed optimization algorithm can find  an $\epsilon$-stationary point with at most $O\pare{{B}/{\epsilon^8}}$ stochastic gradient evaluations. Given that complexity increases with respect to $B$, it becomes crucial to have an appropriately sized sample of Rademacher variable.\vspace{-1mm}

\subsection{Multi-class}\label{subsec:multi-class}



The proposed regularization idea  can also be generalized to multi-class classification.
If the model $f(\cdot)$ is a vector-valued function, i.e., in multi-class classification, $f(h(\bx)): \cX \mapsto \R^o$, we can apply the following \textit{vector-valued contraction lemma} of Rademacher complexity:
\begin{lemma}~\citep{maurer2016vector}
Let $\phi(\cdot,\cdot): \R^o \mapsto \R$ be $G_\phi$-Lipschitz in the first argument, and $f(\cdot): \R^{d'} \mapsto \R^o$ be vector-valued function. Then, the following facts hold true for Rademacher complexity over $\cF$ and any $h: \cX \mapsto \R^{d'}$:
\begin{align*}
   & \E_{\boldsymbol{\varepsilon} } \left[\sup_{f\in\cF}\frac{1}{n} \sum_{i=1}^n \varepsilon_i \phi( f(h(\bx_i)) ,y_i)\right] \\
    &\leq \sqrt{2} G_\phi \E_{\boldsymbol{\varepsilon}_i} \left[\sup_{f\in\cF}\frac{1}{n} \sum_{i=1}^n \boldsymbol{\varepsilon}_i   f(h(\bx_i)) \right],
\end{align*}
where $\boldsymbol{\varepsilon}_i \in \{-1,+1\}^o$ is Rademacher vector.
\end{lemma}
Now the empirical minimization problem becomes:
  \begin{align*}
     \min_{\bw \in \cW} \cL_{\hat{\cU}} (\bw) +\lambda \frac{1}{B} \sum_{j=1}^B \left[\max_{\mV \in \cV}\frac{1}{n} \sum_{i=1}^n (\bsigma_i^j)^\top f( h(\bx_i))  \right] .
 \end{align*}
 \textit{Specifically,} if the top layer is a linear projection layer, i.e., $f(h(\bx)) = \mV h(\bx)$, the objective is equivalent to:
 {\small \begin{align*} 
    \min_{\bw \in \cW} \cL_{\hat{\cU}} (\bw) +\lambda \frac{1}{B} \sum_{j=1}^B \left[\max_{\mV \in \cV} \text{tr}\left(\mV \frac{1}{n} \sum_{i=1}^n h(\bx_i)(\bsigma_i^j)^\top \right) \right],  
 \end{align*}}
where $\text{tr}(\cdot)$ denotes the trace of a matrix. Here the inner problem is convex and easy to solve with simple (stochastic) gradient ascent.


 \section{Experiments} \label{sec short exp}\vspace{-2mm}
In this section, we empirically evaluate the proposed regularization method in improving the generalization of unsupervised pre-training to downstream tasks. 
We utilize the Masked AutoEncoder (MAE)~\citep{he2022masked} as the base unsupervised pre-training method. 
\begin{figure*}[t]
    \centering
    \includegraphics[width=0.75\textwidth]{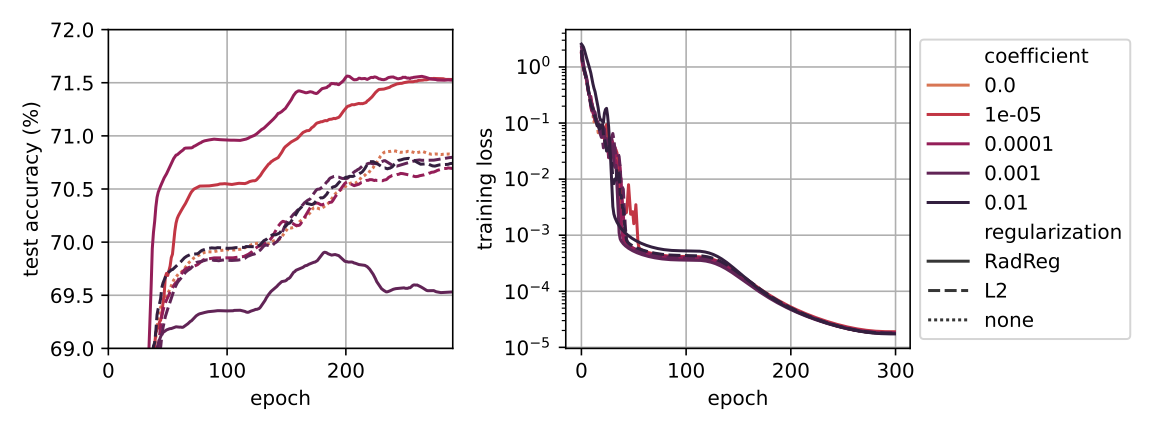}\vspace{-0.5cm}
    \caption{Testing and training accuracy by epochs, averaged by three repetitions.} \vspace{-0.25cm}
    \label{fig:ft_curves}
\end{figure*}
We conduct experiments using 50,000 images from CIFAR10 dataset~\citep{krizhevsky2009learning} for pre-training and 4,096 few-shot STL~\citep{coates2011analysis} samples for finetuning. Since our regularization requires unlabeled data from the downstream task, but L2 and non-regularization methods cannot leverage those data, for a fair comparison, we incorporate the fine-tuning data into a separate unsupervised loss with the same formulation as the MAE loss:
\begin{align*}
  &\min_{g,h}  \cL_{\widehat\cU} (g\circ h)  + \alpha \cdot \cL_{\widehat\cD}  (g\circ h)+ 
  \lambda \frak{R}_{\widehat\cT} (\cF\circ h) \   \text{(RadReg)},\\
  &\min_{g,h}  \cL_{\widehat\cU} (g\circ h)  + \alpha \cdot \cL_{\widehat\cD}  (g\circ h)+ 
  \lambda \norm{\mW}^2 \qquad \quad  \text{($L_2$)},\\
  &\min_{g,h}  \cL_{\widehat\cU} (g\circ h)  + \alpha \cdot \cL_{\widehat\cD}  (g\circ h) \ \qquad   \text{(Non-regularized)},
\end{align*}
where we assume $h$ is parameterized by $\mW$ and $\alpha$ is fixed as 0.01. Our proposed regularization will further leverage the data to control the complexity of learned representations.
The details of experiments are included in  Appendix~\ref{sec:app:exp}.

\begin{table}
\begin{subtable}[t]{0.48\textwidth}
    \setlength{\tabcolsep}{1.0mm}
    \begin{tabular}{cc|ccc}
    \toprule
    \textbf{Reg.} & ${\lambda}$ & \textbf{Final Acc} & \textbf{Best Acc} & \textbf{Train Acc} \\
    \midrule
     None & - & 70.8 (0.2) & 70.9 (0.3) & 100 (0.) \\
     \midrule
      & $10^{-4}$ & 70.7 (0.3) & 70.7 (0.3) & 100 (0.) \\
     $L_2$   & $10^{-3}$ & 70.8 (0.6) & 70.9 (0.5) & 100 (0.) \\
        & $10^{-2}$ & 70.7 (0.6) & 70.9 (0.5) & 100 (0.) \\
    \midrule
    \multirow{3}{*}{RadReg}   & $10^{-5}$ & \textbf{71.5} (0.2) & \textbf{71.8} (0.3) & 100 (0.) \\
        & $10^{-4}$ & 71.5 (0.4) & 71.6 (0.7) & 100 (0.) \\
     & $10^{-3}$ & 69.6 (0.6) & 69.6 (0.5) & 100 (0.) \\  
    \bottomrule
    \end{tabular}
    \caption{End-to-end fine-tuning}
\end{subtable}
\hspace{\fill}
\begin{subtable}[t]{0.48\textwidth}
  
    \setlength{\tabcolsep}{1.0mm}
    \flushright
    \begin{tabular}{cc|ccc}
    \toprule
    \textbf{Reg.} & ${\lambda}$ & \textbf{Final Acc} & \textbf{Best Acc} & \textbf{Train Acc} \\
    \midrule
     None & - & 55.3 (0.1) & 55.5 (0.0) & 65.7 (0.2) \\
     \midrule
     \multirow{3}{*}{$L_2$}   & $10^{-5}$ & 55.3 (0.1) & 55.5 (0.0) & 57.5 (0.1) \\
        & $10^{-4}$ & 55.9 (0.1) & 56.0 (0.1) & 58.2 (0.1) \\
        & $10^{-3}$ & 54.4 (0.0) & 54.5 (0.0) & 58.2 (0.1) \\
    \midrule
    \multirow{2}{*}{RadReg} & $10^{-5}$ & \textbf{56.8} (0.0) & \textbf{56.9} (0.1) & 65.7 (0.2) \\
        & $10^{-4}$ & 26.8 (0.0) & 27.0 (0.1) & 40.6 (0.1) \\
    \bottomrule
    \end{tabular}
    \caption{Linear fine-tuning}
\end{subtable}
\caption{Evaluation of MAE. 
    Average fine-tuning accuracy is reported with its standard deviations in brackets.}\vspace{-5mm}
\label{tab:mae}
\end{table}

In Table~\ref{tab:mae}, we compare our method to non-regularized MAE training and the one with $L_2$ regularization. We repeat the fine-tuning three times by randomly selecting 4096 samples from the preset STL10 training set, and report the mean and standard deviations.
We vary the coefficient for our and $L_2$ regularization and compare the best test accuracy on fine-tuning.
We observe that our method can effectively improve the downstream performance as early as in the pre-training stage without using any labels.
Compared to $L_2$ regularization, our method can achieve higher test accuracy.

In Figure~\ref{fig:ft_curves}, we show the learning curves by different regularization strategies.
Due to the large capacity of the pre-trained ViT encoder, all methods can sufficiently fit the training set approaching 100\% training accuracy, but the testing accuracy reaches the ceiling.
Our method can improve the best test accuracy by limiting the representation complexity as early as the pre-training stage.
Our method also improves the convergence rate at fine-tuning, when our method reaches the 71\% test accuracy at epoch 80 but the best baseline reaches the same accuracy after 200 epochs.

 \section{Conclusion}\vspace{-2mm}
This paper establishes a generic  learning bound in unsupervised representation pre-training and fine-tuning paradigm. We discover that the generalization depends on representation transferrbaility, representation-induced Rademacher complexity, task heterogeneity and generalization of pre-training task. We apply our theory to analyze the generalization of {CE and MAE} pre-training. Motivated by our theory, we propose Rademacher representation regularization, with a provable convergence guarantee. The experiments validate the superiority of our algorithm. As a future direction,  it would be interesting to expand our analysis to end-to-end model fine-tuning, where task specific head and encoder are jointly updated in fine-tuning stage. 


\section*{Acknowledgement}

The work of YD and MM was partially supported by NSF CAREER Award \#2239374 and NSF CNS   Award \#1956276. 
JZ was supported by NSF \#IIS-2212174 and IIS-1749940 and NIA \#1RF1AG072449.

\clearpage

\bibliographystyle{plainnat}
\bibliography{references.bib}
\clearpage
\section*{Checklist}

 \begin{enumerate}

 \item For all models and algorithms presented, check if you include:
 \begin{enumerate}
   \item A clear description of the mathematical setting, assumptions, algorithm, and/or model. [Yes]
   \item An analysis of the properties and complexity (time, space, sample size) of any algorithm. [Yes]
   \item (Optional) Anonymized source code, with specification of all dependencies, including external libraries. [Not Applicable]
 \end{enumerate}

 \item For any theoretical claim, check if you include:
 \begin{enumerate}
   \item Statements of the full set of assumptions of all theoretical results. [Yes]
   \item Complete proofs of all theoretical results. [Yes]
   \item Clear explanations of any assumptions. [Yes]     
 \end{enumerate}

 \item For all figures and tables that present empirical results, check if you include:
 \begin{enumerate}
   \item The code, data, and instructions needed to reproduce the main experimental results (either in the supplemental material or as a URL). [Yes]
   \item All the training details (e.g., data splits, hyperparameters, how they were chosen). [Yes]
         \item A clear definition of the specific measure or statistics and error bars (e.g., with respect to the random seed after running experiments multiple times). [Yes]
         \item A description of the computing infrastructure used. (e.g., type of GPUs, internal cluster, or cloud provider). [No]
 \end{enumerate}

 \item If you are using existing assets (e.g., code, data, models) or curating/releasing new assets, check if you include:
 \begin{enumerate}
   \item Citations of the creator If your work uses existing assets. [Yes]
   \item The license information of the assets, if applicable. [Not Applicable]
   \item New assets either in the supplemental material or as a URL, if applicable. [Not Applicable]
   \item Information about consent from data providers/curators. [Not Applicable]
   \item Discussion of sensible content if applicable, e.g., personally identifiable information or offensive content. [Not Applicable]
 \end{enumerate}

 \item If you used crowdsourcing or conducted research with human subjects, check if you include:
 \begin{enumerate}
   \item The full text of instructions given to participants and screenshots. [Not Applicable]
   \item Descriptions of potential participant risks, with links to Institutional Review Board (IRB) approvals if applicable. [Not Applicable]
   \item The estimated hourly wage paid to participants and the total amount spent on participant compensation. [Not Applicable]
 \end{enumerate}

 \end{enumerate}

 \newpage
\appendix
\clearpage
 \appendix
\onecolumn
\paragraph{Organization} The appendix is organized as follows. In Appendix~\ref{app:basic ineq} we will introduce some helper inequalities that we will be utilized in our proofs and prove the main generalization theorem in Appendix~\ref{app:main:thm:proof}. In Appendices~\ref{sec:app:ce} and~\ref{app:mae} we will provide the proofs in Section~\ref{sec:dae} (generalization of pre-training with a context encoder) and Section~\ref{sec:mask} (generalization of pre-training with masked autoencoder with a transformer), respectively. In Appendix~\ref{app:proof convergence}, we  provide the proof of convergence of the proposed algorithm in Section~\ref{sec:radreg}. At last, in Appendix~\ref{sec:app:exp} we will provide the   details of setup for our experiments.

\section{Basic Inequalities  }\label{app:basic ineq}
In this section, we provide  some general technical  results that will be used in our proofs. 

\begin{proposition}[Total variation distance and $L_1$ distance]~\cite[Proposition~4.2]{levin2017markov}
Given two probability measures $\mathcal{P}$ and $\cQ$ defined over instance space $\cX$, the following inequality holds:
\label{prop:TV distance}
    \begin{align*}
        \norm{\mathcal{P}- \cQ}_{\mathrm{TV}} = \frac{1}{2}\sum_{\bx\in\cX} |\mathcal{P}(\bx) - \cQ(\bx)|. 
    \end{align*}
\end{proposition}

\begin{proposition}[Ruhe’s trace inequality]~\citep{Ruhe1970PerturbationBF}
If $\mA$ and $\mB$ are positive semidefinite Hermitian matrices with eigenvalues,
\begin{align}
    a_1 \geq ... \geq a_n \geq 0, \ b_1 \geq ... \geq b_n \geq 0,
\end{align}
repsectively, then
\label{prop:trace ineq}
    \begin{align*}
         \sum_{i=1}^n a_{i} b_{n-i+1} \leq \trace\pare{\mA\mB} \leq \sum_{i=1}^n a_{i} b_{i}.
     \end{align*}
\end{proposition}

\section{Proof of Main Generalization Theorem} \label{app:main:thm:proof}
In this section we provide the proof of main result on generalization of fine-tuned model composed with an unsupervised pre-trained model stated in  Theorem~\ref{thm:main generalization}. For readability purposes, we re-state the theorem here:

 \begin{theorem}[Theorem~\ref{thm:main generalization} restated]\label{app:thm:main generalization}
 Assume $\hat{h}$ and $\hat{g}$ are the pre-trained representation function and its associated decoder function, and real valued non-negative loss $\phi$ to be $G_\phi$ Lipschitz and bounded by $B_\phi$. Assume pre-training and fine-tuning task admit a $(C_\beta, \beta)$ representation transferrability on $\hat{h}$ and $h^*_\cU$ .  If we solve (\ref{eq:fine-tune obj}) to get $\hat{f}$, then with probability at least $1-\nu$, the following statement holds  
\begin{align*}
   \mathcal{E}_{\cT}(\hat{f}, \hat{h})    &\leq C_\beta \left(   \cE_{\cU}(\hat{g}, \hat{h}) \right)^\beta +4G_\phi {\frak{R}}_{\widehat \cT}( \cF\circ \hat{h})  + 4B_\phi\sqrt{\frac{\log(1/\nu)}{n}}  + 4B_\phi \norm{\cT-\cU_{\cX}}_{\mathrm{TV}}   +   \min_{f\in\cF}\cE_{\cT}(f,  h^*_{\cU} ),
\end{align*}
where $h^*_{\cU} = \arg\min_{h \in \cH}  \min_{g\in\mathcal{G}} \cL_{\cU}(g\circ h)$ is the optimal pre-training representation function, and $\norm{\sP-\cQ}_{\mathrm{TV}} = \sup_{A \in \Omega} |\sP(A)-\cQ(A)|$ denotes total variation distance between two distributions.

 \end{theorem}

\subsection{Proof of Theorem~\ref{thm:main generalization}}\label{app:proof main generalization}
\begin{proof}
For the ease of presentation we define 
\begin{align*}
f^*_{\cT}(h) = \arg\min_{f\in\cF} \cR_{\cT}(f\circ h) :=  \E_{\bx\sim \cT}[\phi(f\circ h(\bx),y(\bx))].
\end{align*}

That is, the optimal fine-tuned risk minimizer in function class  $\mathcal{F}$ w.r.t. distribution $\mathcal{T}$ over domain,  given a representation function $h$, which denotes the optimal risk minimizer for downstream task with labeling function $y(\cdot)$, for a given representation function.   Also, recall $h^*_{\cU} = \arg\min_{h \in \cH}  \min_{g\in\mathcal{G}} \cL_{\cU}(g\circ h)$ denotes the optimal pre-training representation function.

By standard risk decomposition we have:

\begin{align*}
   \mathcal{E}_{\cT}(\hat{f}, \hat{h}) &=  \cR_{\cT}(\hat{f} \circ \hat{h})-\min_{f\in\cF,h\in\cH}\cR_{\cT}(f  \circ h ) \\
   &= \cR_{\cT}(\hat{f} \circ \hat{h})- \min_{f\in\cF}\cR_{\cT}(f \circ \hat{h}) + \min_{f\in\cF} \cR_{\cT}(f \circ  \hat{h}) - \min_{f\in\cF,h\in\cH}\cR_{\cT}(f  \circ h ) \\
 &= \underbrace{\cR_{\cT}(\hat{f} \circ \hat{h})- \min_{f\in\cF}\cR_{\cT}(f \circ \hat{h}) }_{\text{\sffamily{I}}}\\
 & \quad + \underbrace{ \min_{f\in\cF} \cR_{\cU_{\cX}}(f \circ  \hat{h}) - \min_{f\in\cF }\cR_{\cU_{\cX}}(f  \circ h^*_{\cU} )}_{\text{\sffamily{II}}}  \\
 &\quad + \underbrace{\pare{\min_{f\in\cF}\cR_{\cT}(f \circ  \hat{h}) - \min_{f\in\cF}\cR_{\cU_{\cX}}(f \circ  \hat{h})}}_{\text{\sffamily{III}}} \\
 & \quad - \underbrace{\pare{\min_{f\in\cF,h\in\cH}\cR_{\cT}(f  \circ h ) - \min_{f\in\cF }\cR_{\cU_{\cX}}(f  \circ h^*_{\cU} ) }}_{\text{\sffamily{IV}}}\\ 
\end{align*}

We now turn to bounding each term in RHS of above inequality. 

\noindent \textbf{Bounding \text{\sffamily{I}}}. The term \text{\sffamily{I}} can be bounded by following standard results in uniform convergence and noting the fact that $\hat{f}$ is empirical risk minimizer of downstream task by fixing the pre-training representation function $\hat{h}$:
\begin{align*}
    \text{\sffamily{I}} &= \cR_{\cT}(\hat{f} \circ \hat{h})- \min_{f\in\cF}\cR_{\cT}(f \circ \hat{h}) \\
    &= \cR_{\cT}(\hat{f} \circ \hat{h}) - \cR_{\widehat\cT}(\hat{f} \circ \hat{h})+\underbrace{\cR_{\widehat\cT}(\hat{f} \circ \hat{h}) - \cR_{\widehat\cT}({f}^*_{\cT}(\hat{h}) \circ \hat{h}) }_{\leq 0}+\cR_{\widehat\cT}({f}^*_{\cT}(\hat{h}) \circ \hat{h})    - \min_{f\in\cF}\cR_{\cT}(f \circ \hat{h})\\
    &\leq 4\frak{R}_{\widehat\cT}(\phi\circ\cF\circ \hat{h}) + 4B_\phi\sqrt{\frac{\log(1/\nu)}{n}}.
\end{align*}

\noindent \textbf{Bounding \text{\sffamily{III}}}. To bound \text{\sffamily{III}}, we define $f^*_\cU (h) = \argmin_{f\in\cF} \cR_{\cU_{\cX}}(f \circ  \hat{h})$, where $ \cR_{\cU_{\cX}}(f\circ h) := \E_{\bx\sim\cU_{\cX}}[\phi(f\circ h(\bx),y(\bx))]$ denotes the  risk realized by pre-training marginal data distribution $\cU_{\cX}$ and downstream labeling function $y(\cdot)$ (Definition~\ref{def:representation:Transferability}). We have:
\begin{align*}
   \text{\sffamily{III}} &= \min_{f\in\cF}\cR_{\cT}(f \circ  \hat{h}) - \min_{f\in\cF}\cR_{\cU_{\cX}}(f \circ  \hat{h})  \leq      \cR_{\cT}(f^*_\cU (\hat{h}) \circ \hat{h}) -  \cR_{\cU_{\cX}}(f^*_\cU (\hat{h}) \circ  \hat{h}) \\
   &= \E_{\bx\sim\cT} [\phi(f^*_\cU (\hat{h}) \circ  \hat{h}(\bx),\by)]  - \E_{\bx\sim\cU} [\phi(f^*_\cU (\hat{h}) \circ  \hat{h}(\bx),\by)]\\
    & =   \sum_{\bx \in \cX} |\cT(\bx) - \cU_{\cX}(\bx)|\cdot \phi(f^*_\cU (\hat{h}) \circ  \hat{h}(\bx),\by)\\
   & \leq B_\phi \sum_{\bx \in \cX} |\cT(\bx) - \cU_{\cX}(\bx)| \\
   & \leq 2B_\phi \norm{\cT-\cU_{\cX}}_{\mathrm{TV}}.
\end{align*}
where the last step follows from  Proposition~\ref{prop:TV distance}.

\noindent \textbf{Bounding \text{\sffamily{IV}}}. For \text{\sffamily{IV}}, recalling that  $f^*_{\cT}(h) = \arg\min_{f\in\cF}\cR_{\cT}(f  \circ h )$, and we have
\begin{align*}
   \text{\sffamily{IV}} &= \min_{f\in\cF }\cR_{\cU_{\cX}}(f  \circ h^*_{\cU} )- \min_{f\in\cF,h\in\cH}\cR_{\cT}(f  \circ h )  \\
   &=  \min_{f\in\cF}\cR_{\cU_{\cX}}(f\circ h^*_{\cU})- \min_{f\in\cF}\cR_{\cT}(f \circ h^*_{\cU})  + \min_{f\in\cF} \cR_{\cT}(f\circ h^*_{\cU})-\min_{f\in\cF,h\in\cH}\cR_{\cT}(f  \circ h ) \\
   &\leq  \cR_{\cU_{\cX}}(f^*_{\cT}( h^*_{\cU})\circ h^*_{\cU})-  \cR_{\cT}(f^*_{\cT}( h^*_{\cU}) \circ h^*_{\cU})  + \min_{f\in\cF} \cR_{\cT}(f\circ h^*_{\cU})-\min_{f\in\cF,h\in\cH}\cR_{\cT}(f  \circ h ) \\
     & \leq 2B_\phi \norm{\cT -\cU_{\cX}}_{\mathrm{TV}}+  \min_{f\in\cF}\cE_{\cT}  ( f,h^*_{\cU}) 
\end{align*}
where at last step we use the same reasoning we used in bounding \text{\sffamily{III}}, and the definition of $\cE_{\cT} (\cdot)$.

\noindent \textbf{Bounding \text{\sffamily{II}}}. It remains to bound \text{\sffamily{II}}. Under the representation transferability assumption, we know
\begin{align*}
    \text{\sffamily{II}} &= \min_{f\in\cF} \cR_{\cU_{\cX}}(f \circ  \hat{h}) - \min_{f\in\cF }\cR_{\cU_{\cX}}(f  \circ h^*_{\cU} )\\
   &   \leq C_\beta \pare{\min_{g\in\cG} \cL_{\cU}(g \circ  \hat{h}) - \min_{g\in\cG }\cL_{\cU}(g \circ h^*_{\cU} ) }^\beta\\
     &   \leq C_\beta \pare{  \cL_{\cU}(\hat{g} \circ  \hat{h}) - \min_{g\in\cG }\cL_{\cU}(g \circ h^*_{\cU} ) }^\beta\\
      &   = C_\beta \pare{  \cL_{\cU}(\hat{g} \circ  \hat{h}) - \min_{g\in\cG,h\in\cH }\cL_{\cU}(g \circ h ) }^\beta\\
   &=  C_\beta \pare{\cE_{\cU}(\hat{g},\hat{h})}^\beta.
\end{align*}
where  the last step follows from the definition of $\cE_{\cU}(\cdot)$.

Putting pieces \text{\sffamily{I}}-\text{\sffamily{IV}} together yields:
\begin{align*}
   \cE_{\cT}(\hat{f},\hat{h})   &\leq C_\beta \left(   \cE_{\cU}(\hat{g}, \hat{h})\right)^\beta +4G_\phi {\frak{R}}_{\widehat \cT}( \cF\circ \hat{h})  + 4B_\phi\sqrt{\frac{\log(1/\nu)}{n}} \\
   &\quad + 4B_\phi \norm{\cT-\cU_{\cX}}_{\mathrm{TV}}   +   \min_{f\in\cF}\cE_{\cT}(f,  h^*_{\cU} ),
\end{align*}
 thus leading to the desired generalization bound stated in Theorem~\ref{thm:main generalization}.

\end{proof}
As mentioned earlier, to instantiate  Theorem~\ref{thm:main generalization} to a particular application, we need to establish bounds on representation
transferrability, generalization of pre-training task, and representation-induced Rademacher complexity as we demonstrate  on two specific pre-training tasks.  We note that similar notions to representation
transferrability were proposed in~\citep{tripuraneni2020theory,ge2023provable,du2020few,zhang2023trade}, but they do not have exponent in definition, so cannot capture the transferrability when pre-training and downstream task losses are not homogeneous.  The term $\min_{f\in\cF}\cE_{\cT}(f,  h^*_{\cU} )$  characterizes how well the optimal pre-training task encoder is when applied on downstream task. It will depend on specific pre-training and downstream distribution. Since we do not make distributional assumption, analyzing this term is beyond the scope of this paper.

\section{Proof of Generalization for Pre-training with Context Encoder}
\label{sec:app:ce}
 In this section we prove the results on generalization of pre-training with Context Encoder (CE) and fine-tuning on binary classification as downstream task provided in Subsection~\ref{sec:dae}. Recall during pre-training, we draw a set of unlabeled data,e.g., images $\cbr{\bz_1,...,\bz_N}$, and corrupt these data to make $\cbr{\tilde\bz_1,...,\tilde\bz_N}$, then a deep neural network is trained by reconstructing the corrupted pixel of a given image. The encoder-decoder architecture is defined as follows:
\begin{align*}
 \text{\sffamily{encoder:}} \quad  & h(\bx) =  \sigma\pare{\mW_L \cdots \sigma \pare{  \mW_1 \bx} }  ,\\
\text{\sffamily{decoder:}} \quad &  g(h(\bx)) =  \mW_{L+1} h(\bx) . 
\end{align*}
where $\mW_1 \in \R^{m\times d}$, $\mW_2,...,\mW_{L} \in \R^{m\times m}$, and $\mW_{L+1} \in \R^{d\times m}$ (for simplicity we assume the hiddent layers share the same dimension $m$). We assume each layer's weight is with bounded norm: $\norm{\mW_l}\leq W(l)$, $\norm{\mW_l}_{2,1}\leq B(l), \forall l \in [L+1]$.
The hypothesis class for encoder is then defined as:
\begin{align*}
    \cH := \cbr{\bx \mapsto \sigma\pare{\mW_L\cdots \sigma\pare{\mW_1\bx}} : \norm{\mW_l} \leq W(l), \norm{\mW_l}_{2,1} \leq B(l), \forall l \in [L]   }
\end{align*}
and decoder class is defined as:
\begin{align*}
    \cG := \cbr{\bx \mapsto \mW_{L+1}\bx: \norm{\mW_{L+1}}\leq W(L+1),\norm{\mW_{L+1}}_{2,1}\leq B(L+1)  }.
\end{align*}
In pre-training stage we optimize the following empirical unsupervised losses: 
{ 
\begin{align}
   \min_{g\in\mathcal{G}, h\in\cH } \cL_{\widehat \cU}(g\circ h)  := \frac{1}{2} \sum_{i=1}^N \norm{g(h(\widetilde \bz_i))  - \bz_i}^2, \label{eq:dae pretrain obj}
\end{align}} 

where $\tilde \bz_i = T_1 (\bz_i)$, and $T_1: \cX \mapsto \cX$ is some random transformation, e.g, rotating, scaling, adding Gaussian noise or masking pixels. 

After pre-training, we discard the top layer of the network, and use the rest layers as an encoder. A linear projection head is added on top of encoder in downstream training:
\begin{align*}
   \text{\sffamily{downstream model:}} \quad f(\hat{h}(\bx)) = \btheta^\top \hat{h}(\bx),
\end{align*}
with $\norm{\btheta}\leq R$,
and we assume that only the linear head is trainable during fine-tune stage.  We optimize a binary classification task with Lipschitz loss function as fine-tuning task:
\begin{align*} 
        \min_{\norm{\btheta}\leq R} \cR_{\widehat\cT}(\btheta \circ \hat h) &= \frac{1}{n}\sum\nolimits_{i=1}^{n} \phi(\btheta^\top h (  {\mathbf{x}}_i ),y_i ),
\end{align*}
to get $\hat{f}$, where $y_i \in \{-1, +1\}$  is binary labeling function for downstream task.

 
\noindent\textbf{Roadmap.}~We will provide proof of Theorem~\ref{thm:DAE main} in the following subsections. The roadmap is as follows:  in Appendix~\ref{app:proof DAE transfer} we first show that the CE pre-training admits bounded representation transferrability to downstream task (the proof of Lemma~\ref{lem: DAE transfer}), and then  in Appendix~\ref{app:proof generalization DAE} we prove the generalization of CE pre-training task (Lemma~\ref{lem: DAE pretrain gen}), and finally in Appendix~\ref{app:proof DAE main} we conclude the proof for Theorem~\ref{thm:DAE main} by showing that the representation-induced Rademacher complexity is bounded.

 \subsection{Proof of Transferability}\label{app:proof DAE transfer}
 In this subsection we provide the proof of Lemma~\ref{lem: DAE transfer}. For notational convenience we define the following quantities: 
\begin{align*}
   & \Delta^{ft}_{\cU}(\hat{h}, h^*_{\cU}) = \min_{\norm{\btheta}\leq R} \E_{(\tilde\bz,\bz) \sim \cU} [\phi(\btheta^\top  \hat{h}(\tilde\bz )   )]-\min_{\norm{\widetilde\btheta}\leq  R} \E_{\tilde\bz  \sim \cU_{\cX}}[\phi(\tilde{\btheta}^\top      {h}^*_{\cU} (\tilde\bz )     )],\\
   & \Delta^{pt}_{\cU} (\hat{h}, h^*_{\cU}) = \min_{\mW_{L+1} }  \E_{(\tilde\bz,\bz) \sim \cU }\norm{  \mW_{L+1} \hat h(\tilde\bz ) -\bz}^2 - \min_{\widetilde\mW_{L+1} } \E_{\tilde\bz  \sim \cU_{\cX} }\norm{  \widetilde\mW_{L+1} h^*_{\cU}(\tilde\bz )- \bz }^2
\end{align*} 
To prove Lemma~\ref{lem: DAE transfer}, we are going to show $\Delta^{ft}_{\cU}(\hat{h}, h^*_{\cU}) \leq C_{\beta} \pare{ \Delta^{pt}_{\cU} (\hat{h}, h^*_{\cU})  }^{\beta}$ holds for some $C_{\beta},\beta$.

\paragraph{Upper bounding $\Delta^{ft}_{\cU}(\hat{h}, h^*_{\cU})$:}

We examine $\Delta^{ft}_{\cU}(\hat{h}, h^*_{\cU})$ first. We define the optimal head for classification task on distribution $\cU_{\cX}$ under represetation $h^*_{\cU}$ as $\tilde\btheta^* = \arg\min_{\norm{\tilde\btheta}\leq   R} \E_{\tilde\bz  \sim \cU_{\cX}}[\phi(\tilde{\btheta}^\top      {h}^*_{\cU} (\tilde\bz )     )]$.
\begin{align*}
    \Delta^{ft}_{\cU}(\hat{h}, h^* ) &=  \min_{\norm{\btheta}\leq R} \E_{\bx \sim \cU}[\phi(\btheta^\top  \hat{h}(\tilde\bz)   )]-  \E_{\bx \sim \cU}[\phi( \tilde{\btheta}^{*\top}     {h}^*_{\cU} (\tilde\bz)     )]\\
    &\leq  \min_{\norm{\btheta}\leq R} \E_{(\tilde\bz,\bz) \sim \cU}\left|  \btheta^\top  \hat{h}(\tilde\bz)   -   \tilde{\btheta}^{*\top}      {h}^*_{\cU} (\tilde\bz)  \right|\\
    &\leq  \min_{\norm{\btheta}\leq R} \sqrt{\E_{(\tilde\bz,\bz) \sim \cU}\pare{ \btheta^\top  \hat{h}(\tilde\bz)    -      \tilde{\btheta}^{*\top}     {h}^*_{\cU} (\tilde\bz)    }^2}\\
    &=  \min_{\norm{\btheta}\leq R} \sqrt{   \btheta^\top \E\left[ \hat{h}(\tilde\bz)\hat{h}^\top(\tilde\bz)\right] \btheta   - 2 \btheta^\top  \E\left[\hat{h}(\tilde\bz) {{h}^*_{\cU}}^\top (\tilde\bz) \right]{\tilde{\btheta}^*}       +  \tilde{\btheta}^{*\top}  \E\left[{h}^*_{\cU} (\tilde\bz)  {{h}^*_{\cU}}^\top (\tilde\bz)\right] {\tilde{\btheta}^*}  }\\
\end{align*}
Since $\sqrt{f(x)}$ and $f(x)$ attain the minimum at the same point, we examine the minimum of $ \btheta^\top \E\left[ \hat{h}(\tilde\bz)\hat{h}^\top (\tilde\bz)\right] \btheta   - 2 \btheta^\top  \E\left[\hat{h}(\tilde\bz) h_{\cU}^{*\top} (\tilde\bz) \right]{\tilde{\btheta}^*}  +  \tilde{\btheta}^{*\top}  \E\left[{h}^*_{\cU}(\tilde\bz)  {h_{\cU}}^{*\top} (\tilde\bz)\right] \tilde{\btheta}^* $ over $  \btheta $. Under unconstrained setting, the minimum of above statement is $  \tilde{\btheta}^{*\top} \Lambda   \tilde{\btheta}^* $

when $ \btheta  = \pare{\E\left[ {\hat{h}}(\tilde\bz) {\hat{h}}^\top(\tilde\bz)\right]}^{\dagger}\E\left[\hat{h}(\tilde\bz) {{h}^*_{\cU}}^\top(\tilde\bz) \right] {\tilde{\btheta}^*}$,   
and
\begin{align*}
    \Lambda =  \E\left[\hat{h}(\tilde\bz)\hat{h}^\top(\tilde\bz)\right] -\E\left[{{h}^*_{\cU}}(\tilde\bz)\hat{h}^\top(\tilde\bz)  \right]\pare{\E\left[{h}^*_{\cU}(\tilde\bz)  {{h}^*_{\cU}}^\top(\tilde\bz) \right] }^{\dagger}\E\left[\hat{h}(\tilde\bz) {{h}^*_{\cU}}^\top(\tilde\bz) \right] . 
\end{align*}
Hence we have 
\begin{align}
    \Delta^{ft}_{\cU}(\hat{h}, h^*_{\cU})  &\leq   \sqrt{  \tilde{\btheta}^{*\top} \Lambda   \tilde{\btheta}^*  } =  \sqrt{\trace(\Lambda \tilde{\btheta}^{*\top} \tilde{\btheta}^*)} \leq \sqrt{d \sigma_{\max}(\Lambda) \sigma_{\max}(\tilde{\btheta}^{*\top} \tilde{\btheta}^*)}, \label{eq: schur complement}
\end{align}
where we applied Ruhe's Trace Inequalities at last step (Proposition~\ref{prop:trace ineq}): $\trace(\mA\mB) \leq \sum_{i=1}^d \sigma_i(\mA)\sigma_i(\mB) \leq d \sigma_{\max}(\mA)\sigma_{\max}(\mB)$.

Finally, we choose large enough $R$ so that we can attain the optimum.

\paragraph{Lower bounding $\Delta^{pt}_{\cU} (\hat{h}, h^* )$}
Now we switch to lower bounding $\Delta^{pt}_{\cU} (\hat{h}, h^* )$. We have:
\begin{align*}
    \Delta^{pt}_{\cU} (\hat{h}, h^*_{\cU})  
    &=  \min_{\mW_{L+1}: \norm{\mW}\leq W(L+1) }  \E_{(\tilde\bz,\bz) \sim \cU }\norm{  \mW_{L+1} \hat h(\tilde\bz) -\bz}^2 -   \E_{(\tilde\bz,\bz) \sim \cU}\norm{ \mW^*_{L+1} h^*_{\cU}(\tilde\bz )- \bz }^2\\
        &=  \min_{\mW_{L+1}: \norm{\mW}\leq W(L+1) }  \E_{(\tilde\bz,\bz) \sim \cU }\norm{  \mW_{L+1} \hat h(\tilde\bz)   -   \mW^*_{L+1} h^*_{\cU}(\tilde\bz )  }^2\\
\end{align*}

where the last step is due to our realizability Assumption~\ref{assumption: realizable}, the optimal encoder-decoder exists in the hypothesis class which can perfectly recover masked data.

Hence 
\begin{align*}
    \Delta^{pt}_{\cU} (\hat{h}, h^*_{\cU})    
     & = \min_{\mW_{L+1} \in \R^{d\times m}}  \E_{(\tilde\bz,\bz) \sim \cU }  \norm{\mW_{L+1} \hat h(\tilde\bz)-  \mW_{L+1}^*  h^*(\tilde\bz)  }^2  \\ 
     &  =  \min_{\bw_r \in \R^m, r \in [d] }  \E_{(\tilde\bz,\bz) \sim \cU }  \sum_{r=1}^d  \norm{ \bw_r^\top \hat h(\tilde\bz)-  {\bw_r^*}^\top  h^*(\tilde\bz)} ^2  \\ 
     & \geq    \sum_{r=1}^d \min_{\bw_r \in \R^m} \E_{(\tilde\bz,\bz) \sim \cU } \norm{ \bw_r^\top \hat h(\tilde\bz)-  {\bw_r^*}^\top  h^*(\tilde\bz)} ^2  \\ 
     & \geq \sum_{r=1}^d \min_{\bw_r \in \R^m } \E_{(\tilde\bz,\bz) \sim \cU } \pare{ \bw_r^\top \hat h(\tilde\bz) {\hat h(\tilde\bz)}^\top \bw_r -2\bw_r^\top \hat h(\tilde\bz) {  h^*_{\cU}(\tilde\bz)}^\top \bw_r^* +  {\bw_r^*}^\top  h^*_{\cU}(\tilde\bz) {h^*_{\cU}(\tilde\bz)}^\top  {\bw_r^*}  }.
\end{align*}

According to similar reasoning in the proof of upper bound, with $\Lambda$ defined in the same way as (\ref{eq: schur complement}), we have 
\begin{align*}
       \Delta^{pt}_{\cU} (\hat{h}, h^*_{\cU})    & \geq   \sum_{r=1}^d {\bw_r^*}^\top  \Lambda\bw_r^*  \\ 
     & =  \textrm{tr}\pare{  \Lambda   \sum_{r=1}^m\bw_r^*{\bw_r^*}^\top} \\
     &\geq  \sigma_{\max}( \Lambda) \sigma_{\min}\pare{  \sum_{r=1}^d\bw_r^*{\bw_r^*}^\top}  
\end{align*}
where at last step we apply  Ruhe's trace inequality (Proposition~\ref{prop:trace ineq})): $\trace\pare{\mA\mB} \geq \sigma_{\max} (\mA)\sigma_{\min}(\mB)$. 
Therefore, we can conclude that
\begin{align*}
   \frac{\Delta^{ft}_{\cU} (\hat{h}, h^*_{\cU}) }{\pare{\Delta^{pt}_{\cU} (\hat{h}, h^*_{\cU})}^{1/2} }  \leq O\pare{\frac{\sqrt{d\sigma_{\max}( \tilde{\btheta}^* \tilde{\btheta}^{*\top} )}}{\sqrt{\sigma_{\min}\pare{  \sum_{r=1}^d \bw_r^* {\bw_r^*}^\top } } } },
\end{align*}

which indicates that Context Encoder pretraining admits an $\pare{\Omega\pare{\frac{\sqrt{d\sigma_{\max}( \tilde{\btheta}^* 
 \tilde{\btheta}^{*\top})}}{\sqrt{\sigma_{\min}\pare{  \sum_{r=1}^d \bw_r^* {\bw_r^*}^\top}} }},\frac{1}{2}}$ representation transferrability to binary classification task. In the main paper Lemma~\ref{lem: DAE transfer} we omit the constant dependency for ease of exposition.

\subsection{Proof of generalization of CE pretraining task}\label{app:proof generalization DAE}
In this section we are going to derive generalization bound of the CE pre-training. The generalization is given in the following lemma:

 \begin{lemma}[Generalization of pre-training task]\label{lem: DAE pretrain gen}
 Let $\hat{g},\hat{h}$ be the solution of~\eqref{eq: DAE pt obj}, and $\tilde\mZ  = [\tilde\bz_{1};\ldots;\tilde\bz_{N}]$ is the concatenated pre-training data. Then with probability at least $0.99$ the following statement holds: 
     \begin{align*}
       \cE_{\cU}(\hat{g},\hat{h}) \leq O\pare{ \frac{\pare{{  \norm{\tilde\mZ}^2\ln(2m^2)}  } \pare{\prod_{l=1}^{L+1} W^2(l)} \pare{\sum_{l=1}^{L+1} (\frac{B(l)}{W(l)})^{\frac{2}{3}} }^3 }{N} }.
     \end{align*}
 \end{lemma}
To prove Lemma~\ref{lem: DAE pretrain gen},
we first introduce the following worst case covering number quantity:

\begin{definition}[$L_2$ covering number]
    Given a hypothesis class $\cH$ and a set of data $\cS = \cbr{\bx_1,...,\bx_N}$, let $h(\mX) = [h(\bx_1);...;h(\bx_N)]$ denote the concatenated output of $N$ points.The the covering number $\cN(\cH(\cS),\epsilon,\norm{\cdot})$ is the least cardinality of set $\cC$, such that for every $h \in \cH$, there exists a $h_\epsilon \in \cC$, and ensures that
    \begin{align*}
         \norm{h(\mX) - h_{\epsilon}(\mX)} \leq \epsilon.
    \end{align*}
\end{definition}

\begin{definition}[$L_\infty$ covering number]
    Given a hypothesis class $\cH$ and a set of data $\cS = \cbr{\bx_1,...,\bx_N}$, the worst case covering number $\cN_\infty(\cH(\cS),\epsilon,\norm{\cdot})$ is the least cardinality of set $\cC$, such that for every $h \in \cH$, there exists a $h_\epsilon \in \cC$, and ensures that
    \begin{align*}
        \max_{i\in[N]} \norm{h(\bx_i) - h_{\epsilon}(\bx_i)} \leq \epsilon.
    \end{align*}
\end{definition}

The following result will relate the Rademacher complexity of the local loss class induced by a hypothesis class $\cH$, to the $L_\infty$ covering number of $\cH$.

\begin{theorem}[{\cite[Theorem 1]{srebro2010smoothness}}]
  \label{thm:rad-smooth}
  Given a non-negative $H$-smooth loss $\ell$ bounded by $b$ and a set of data pairs $\widehat \cS = \{(\bx_i,\by_i)\}_{i=1}^N$, Define a local loss class $\cL(r) = \cbr{(\bx,\by) \mapsto \ell(h(\bx),\by) : h\in\cH, \cL_{\widehat{\cS}} (h) \leq r  }$ for some $0\leq r <\infty$.
  Then , for all $f \in \cF$ simultaneously
  \begin{align*}
    \frakR_{\widehat\cS}\pare{\cL(r)} \leq \inf_{\alpha}\pare{\frac{\alpha}{\sqrt{N}} + \int_{\alpha}^{\sqrt{br}} \sqrt{\frac{\ln \cN_{\infty}(\cH,\frac{\epsilon}{\sqrt{12Hr}},\norm{\cdot})}{N}}d\epsilon  }
  \end{align*}
  where the empirical Rademacher complexity of loss class is defined as
  \begin{align*}
    \frakR_{\widehat\cS}\pare{\cL(r)} =  \E_{\boldsymbol \varepsilon} \left[\sup_{h \in \cH, \cL_{\widehat\cS}(h) \leq r} \abs{\frac1n \sum_{i=1}^n \varepsilon_i \ell(h(\bx_i),\by_i)}\right]~.
    \tag{$\varepsilon_1, \ldots, \varepsilon_n \stackrel{\mathrm{iid}}{\sim} \unif\{\pm 1\}$} 
  \end{align*} 
\end{theorem}

The above theorem relates the complexity of loss class to the worst case spectral covering number of function class, in our case, vector valued neural networks. Hence, it remains to find worst case ($L_{\infty}$) covering number of our encoder class
\begin{align}
    \cG\circ \cH :=  \cbr{ \bx \mapsto \mW_{L+1}\sigma\pare{\mW_L \cdots \sigma(\mW_1\bx)}: \|\mW_l\| \leq W(l), \|\mW_l\|_{2,1} \leq B(l) \ \forall l \in [L+1]}. \label{eq:nn class}
\end{align}
 
\begin{lemma}[Implication of~{\cite[Theorem 3.3]{bartlett2017spectrally}}]
Given a set of data pairs $\widehat \cS = \{\tilde\bz_i \}_{i=1}^N$, and hypothesis class defined in (\ref{eq:nn class}), then the following statement holds:
\begin{align*}
    \ln \cN_{\infty}(\cG\circ\cH(S), \epsilon, \norm{\cdot} ) \leq \ln \cN(\cG\circ\cH(S), \epsilon, \norm{\cdot} ) \leq \pare{\frac{  \norm{\tilde\mZ}^2\ln(2m^2)} {\epsilon^2}} \pare{\prod_{l=1}^{L+1} W^2(l)} \pare{\sum_{l=1}^{L+1} (\frac{B(l)}{W(l)})^{\frac{2}{3}} }^3.
\end{align*}
where $\tilde\mZ = [\tilde\bz_1;...;\tilde \bz_N]$.
\end{lemma}
\begin{proof}
We define $g\circ h(\mX) = [g(h(\bx_1));...;g(h(\bx_N))] \in \R^{N\times d}$.
    Notice the fact that 2-norm of a row of a matrix, is always less than the spectral norm of the matrix:
    \begin{align*}
     \max_{i\in[N]} \norm{g\circ h(\bx_i) - g'\circ h'(\bx_i)}  \leq \max_{\norm{\ba} \leq 1} \norm{ (g\circ h(\mX) - g'\circ h'(\mX))^\top \ba}  = \norm{g\circ h(\mX) - g'\circ h'(\mX)},
    \end{align*}
    hence we can have the following fact for covering numbers:
    \begin{align}
         \ln \cN_{\infty}(\cG\circ\cH(S), \epsilon, \norm{\cdot} ) \leq \ln \cN(\cG\circ\cH(S), \epsilon, \norm{\cdot} ). \label{eq:covering fact}
    \end{align}
    At last plugging the bound for $\ln \cN(\cG\circ\cH(S), \epsilon, \norm{\cdot} )$ from~\citep{bartlett2017spectrally} concludes the proof.
\end{proof}
Equipped with above results, we are ready to show the local Rademacher complexity of loss class induced by encoder-decoder function class $\cG\circ \cH$:
\begin{lemma}\label{lem:local rademacher bound NN}
Given a hypothesis class $\cH$, if the logarithm of its $L_\infty$ covering number $ \ln \cN_{\infty}(\cG\circ\cH(S), \epsilon, \norm{\cdot} ) $ is bounded by $\frac{c}{\epsilon^2}$, then the following bound for local Rademcaher complexity holds true:
     \begin{align*}
    \frakR_{\widehat\cS}\pare{\cL(r)}    
      &\leq 10\sqrt{\frac{ c Hr}{N}} + 10\sqrt{\frac{cHr}{N  }}\pare{\ln{\sqrt{br}} - \ln\pare{\frac{5}{2}\sqrt{\frac{cHr }{N}}}}.
  \end{align*}
  \begin{proof}
  According to Theorem~\ref{thm:rad-smooth}
we have
    \begin{align*}
    \frakR_{\widehat\cS}\pare{\cL(r)} &\leq   4 {\alpha}  +10 \int_{\alpha}^{\sqrt{br}} \sqrt{\frac{\ln \cN_{\infty}(\cH,\frac{\epsilon}{\sqrt{12Hr}},N)}{N}}d\epsilon   \\
    &\leq 4 {\alpha} + 10\int_{\alpha}^{\sqrt{br}} \sqrt{\frac{ cHr}{N \epsilon^2}}d\epsilon \\
      &\leq 4 {\alpha} + 10\sqrt{\frac{ cHr}{N  }}(\ln{\sqrt{br}} - \ln(\alpha)).
  \end{align*}
  Choosing $\alpha = \frac{5}{2}\sqrt{\frac{B^2cHr}{N}} = \frac{5}{2\sqrt{N}}\sqrt{  c Hr}$ will minimize above bound, and yields:
  \begin{align*}
      \frakR_{\widehat\cS}\pare{\cL(r)}    
      &\leq 10\sqrt{\frac{ c Hr}{N}} + 10\sqrt{\frac{  c Hr}{N  }}\pare{\ln{\sqrt{br}} - \ln\pare{\frac{5}{2}\sqrt{\frac{Hr\cdot c}{N}}}}.
  \end{align*}
  \end{proof}
\end{lemma}

The following theorem connects local Rademacher complexity to population risk.

\begin{theorem}\cite[Theorem 6.1]{Bousquet2002concentration}
\label{thm:local rademacher}
 Given a loss class $\cL(r)$, let $\phi(r)$ be the function such that
    \begin{align*}
         \frakR_{\widehat\cS}\pare{\cL(r)} \leq \phi(r).
    \end{align*}
    then with probability at least $1-\exp(-\nu)$,
    \begin{align*}
        \cL_{\cS} (h) \leq \cL_{\widehat \cS} (h) + 45r^* + \sqrt{\cL_{\cS} (h)}\pare{\sqrt{8r^*_n} + \sqrt{\frac{4b(\log(1/\nu)+6\log\log N)}{N}} }  + 20\frac{b(\nu+6\log\log N)}{N}
    \end{align*}
    where $r^*$ is the largest solution such that $\phi(r) = r$.
\end{theorem}

\subsubsection{Proof of Lemma~\ref{lem: DAE pretrain gen}}
\begin{proof}
    First we evoke Lemma~\ref{lem:local rademacher bound NN} with $c = 12 \norm{\tilde\mZ}^2\ln(2m^2)  \pare{\prod_{l=1}^{L+1} W^2(l)} \pare{\sum_{l=1}^{L+1} (\frac{B(l)}{W(l)})^{\frac{2}{3}} }^3$
    \begin{align*}
      \frakR_{\widehat\cS}\pare{\cL(r)}    
      &\leq 10\sqrt{\frac{Hr\cdot c}{N}} + 10\sqrt{\frac{ cHr}{N  }}\pare{\ln{\sqrt{br}} - \ln\pare{\frac{5}{2}\sqrt{\frac{Hr\cdot c}{N}}}}\\
      & = 10\sqrt{\frac{Hr\cdot c}{N}} + 10\sqrt{\frac{ cHr}{N  }}\ln{\pare{\frac{2}{5}\sqrt{\frac{bN}{Hc}}} }\\
  \end{align*}
     We set $\phi(r) = 10\sqrt{\frac{Hr\cdot c}{N}} \cdot \max\left\{ 1,  \ln\pare{\frac{2}{5}\sqrt{\frac{bN}{Hc}}}\right\}$. Solving the follwoing equation to get $r^*$
    \begin{align*}
       \phi(r) &= 10\sqrt{\frac{Hr\cdot c}{N}} \cdot \max\left\{ 1,  \ln\pare{\frac{2}{5}\sqrt{\frac{bN}{Hc}}}\right\}  = r,\\
    \Longleftrightarrow r^* &= 100 {\frac{H\cdot c}{N}} \cdot \max\left\{ 1,  \ln\pare{\frac{2}{5}\sqrt{\frac{bN}{Hc}}}\right\}^2
    \end{align*}
    Now, according to Theorem~\ref{thm:local rademacher}, and the fact that 
    \begin{align*}
        A \leq B + C\sqrt{A} \Longrightarrow A \leq B+C^2 + \sqrt{B}C,
    \end{align*}
    we have
     \begin{align*}
        \cL_{\cU} &(g\circ h) \leq \cL_{\widehat \cU} (g\circ h) + 45r^* +  \pare{\sqrt{8r^* } + \sqrt{\frac{4b(\log(1/\nu)+6\log\log N)}{N}} }^2  \\
        &\quad + 20\frac{b(\nu+6\log\log N)}{N} + \sqrt{\cL_{\widehat \cU} (g\circ h) + 45r^* +20\frac{b(\nu+6\log\log N)}{N} }\pare{\sqrt{8r^* } + \sqrt{\frac{4b(\log(1/\nu)+6\log\log N)}{N}} }.
    \end{align*}
    Plugging $r^*$, and empirical risk minimizers $\hat g, \hat h$  will conclude the proof.
\end{proof}

 \subsection{Proof of Theorem~\ref{thm:DAE main}} \label{app:proof DAE main}
 
\begin{proof}
Recall that in Theorem~\ref{thm:main generalization}, the generalization bound is given by
\begin{align*}
   \mathcal{E}_{\cT}(\hat{f}, \hat{h})    &\leq C_\beta \left(   \cE_{\cU}(\hat{g}, \hat{h}) + \mu\right)^\beta +4G_\phi {\frak{R}}_{\widehat \cT}( \cF\circ \hat{h})  + 4B_\phi\sqrt{\frac{\log(1/\nu)}{n}}  + 4B_\phi \norm{\cT-\cU_{\cX}}_{\mathrm{TV}}  +   \min_{f\in\cF}\cE_{\cT}(f,  h^*_{\cU} ).
\end{align*}
    Since in the previous subsection we prove the bounded transferrability and generalization of pre-training task, it remains to show the upper bound of representation-induced Rademacher complexity. To this end, we have
    \begin{align*}
          {\frak{R}}_{\widehat \cT}(\phi\circ \mathcal{F}\circ \hat{h})  &= \mathbb{E}_{\boldsymbol \varepsilon \in \{\pm1\}^n} \left[\sup_{\btheta: \norm{\btheta}\leq R} \frac{1}{n}\sum\nolimits_{i=1}^{n} \varepsilon_i   \phi(\btheta^\top \hat{h} (  \mathbf{x}_i ), y_i) \right] \\
          &\leq R G_\phi \mathbb{E}_{\boldsymbol \varepsilon \in \{\pm1\}^n} \left[\sup_{\btheta: \norm{\btheta}\leq R} \frac{1}{n}\sum\nolimits_{i=1}^{n} \varepsilon_i    \btheta^\top \hat{h} (  \mathbf{x}_i )  \right]\\
          &=  \frac{RG_\phi}{n} \mathbb{E}_{\boldsymbol \varepsilon  } \left\|\sum\nolimits_{i=1}^{n} \varepsilon_i   \hat{h} (  \mathbf{x}_i )  \right\|\\
           &\leq \frac{RG_\phi}{n} \sqrt{\mathbb{E}_{\boldsymbol \varepsilon  } \left\|\sum\nolimits_{i=1}^{n} \varepsilon_i   \hat{h} (  \mathbf{x}_i )  \right\|^2}\\
           &= \frac{RG_\phi}{n} \sqrt{\sum\nolimits_{i=1}^{n} \left\|  \hat{h} (  \mathbf{x}_i )  \right\|^2}\\
    \end{align*}
    where at first inequality we apply  Ledoux-Talagrand’s inequality to peel off Lipschitz loss $\phi(\cdot)$, and at last inequality we use the fact that $\varepsilon_i$ are i.i.d. with zero mean, so that the cross terms disappear. For each $\norm{\hat{h} (  \mathbf{x}_i )}^2$, we have:
    \begin{align*}
       \norm{\hat{h} (  \mathbf{x}_i )}^2 \leq \prod_{l=1}^{L+1} W^2(l) \norm{\bx_i}^2,
    \end{align*}
    hence we arrive at 
    \begin{align*}
         {\frak{R}}_{\widehat \cT}(\phi\circ \mathcal{F}\circ \hat{h}) \leq \frac{R G_\phi \sqrt{\prod_{l=1}^{L+1} W^2(l) \sum_{i=1}^n \norm{\bx_i}^2}}{n}.
    \end{align*}
  Plugging Lemmas~\ref{lem: DAE transfer} and~\ref{lem: DAE pretrain gen} back into Theorem~\ref{thm:main generalization}  as well as above bound will complete the proof of Theorem~\ref{thm:DAE main}.
\end{proof}

\section{Proof of Pre-training with Masked Autoencoder with Tranformer Models} \label{app:mae}

We turn to proving the generalization of pretraining with  masked autoencoder (MAE) with tranformer models (Section~\ref{sec:mask}). 

Recall, in MAE pre-training for vision tasks as an example,  we draw a large set of images $\mZ_1,...,\mZ_N \in \R^{K\times d}$, and then randomly mask some patches of each image to get $\tilde{\mZ}_1,...,\tilde{\mZ}_N\in \R^{K\times d}$. Then an encoder-decoder model is trained by recovering the missing  patches (e.g., by utilizing MSE loss $ \ell( \hat\mZ,\mZ) =  \norm{\hat\mZ - \mZ}^2_{\mathrm{F}}$  as pre-training loss).


We will consider $L$-layer transformer as the pre-train encoder model, a single self-attention layer transformer as the pre-train decoder model, and a linear projection layer for binary classification as fine-tune model.

 \paragraph{Encoder Architecture}
 In a $L$-layer transformer, given a input $\mX$, the $l$th layer's output is define as:
\begin{align*}
\mX^{l} = \begin{cases}
    \mX, &l = 0\\
    \mathrm{SA}_{\mW^l}(\mX^{l-1}), &l = [L],  
\end{cases}
\end{align*}
where $\mathrm{SA}_{\mW^l}(\cdot)$ is the $l$-layer self attention module given a collection of weight matrices ${\mW^l} = \pare{\mW^l_{V}, \mW^l_{K}, \mW^l_{Q}, \mW^l_{\mathrm{FC1}}, \mW^l_{\mathrm{FC2}} } \in \R^{d\times d} \times \R^{d\times d_K} \times \R^{d\times d_K} \times \R^{d \times m} \times \R^{m\times d}$   defined as:
\begin{align*}
    \mathrm{SA}_{\mW^l}(\mX^{l-1}) &= \alpha_2\sigma\pare{\mZ^l\mW^l_{\mathrm{FC1}}}\mW^l_{\mathrm{FC2}}+\mZ^l,\\
    \mZ^l &= \pare{\alpha_1\mA^l+\mX^{l-1}}, \\
    \mA^l &= \mathrm{softmax}\pare{\frac{1}{\sqrt{d_K}} \mX \mW^l_K (\mX\mW^l_Q)^\top } \mX \mW^l_V,
    \end{align*}
where $\alpha_1, \alpha_2$ are some small constant, as used in practice~\citep{noci2022signal}.
We use the $L$th layer's output as the final output of encoder, i.e., $h(\mX) = \mX^L$. 
\begin{align*}
   \text{\sffamily{encoder:}} & \quad h(\mX) = \mX^L.
\end{align*}

The hypothesis class of encoder is defined as:
\begin{align}
    \cH = \cbr{
    \begin{aligned}
    \mX \mapsto \mathrm{SA}_{\mW^L}&\pare{\mathrm{SA}_{\mW^{L-1}}...\mathrm{SA}_{\mW^{1}}(\mX)}: \\
    &\norm{\mW^l_{\mathrm{FC1}} },\norm{\mW^l_{\mathrm{FC2}} },\norm{\mW^l_{K} },\norm{\mW^l_{Q} },\norm{\mW^l_{V} } \leq W(l),\\ &\norm{\mW^l_{\mathrm{FC1}}}_{2,1},\norm{\mW^l_{\mathrm{FC2}}}_{2,1},\norm{\mW^l_{K}}_{2,1},\norm{\mW^l_{Q}}_{2,1},\norm{\mW^l_{V}}_{2,1} \leq B(l), \forall l \in [L]
    \end{aligned}
    }. \label{eq:transformer class}
\end{align}

 \paragraph{Decoder Architecture}
 When encoder finished processing masked sequence, we will send the encoder output $h(\tilde \mZ)$ to decoder. The decoder is a simple linear projection layer:
    \begin{align*}  
   \text{\sffamily{decoder:}} & \quad g(h(\tilde \mZ)) =  h(\tilde \mZ) \mW^D,  
\end{align*}

To learn the representation model,
we  solve the following\footnote{In some implementation of MAE pre-training, the MSE loss is not computed on full patches, but only the masked patches. It can be adapted by changing our objective to $\frac{1}{2} \sum_{i=1}^N \norm{\mA \odot (g(h(\widetilde \mZ_i))  - \mZ_i)}^2_{\mathrm{F}}$ where $\mA \in \R^{K \times d}$ is the indicator matrix with $j$ row to be $\mathbf 1$ if $j$th patch is masked, otherwise $\mathbf 0$. This adaptation will not affect our analysis significantly.}:{ 
\begin{align}
   \min_{g\in\mathcal{G}, h\in\cH } \cL_{\widehat \cU}(g\circ h)  := \frac{1}{2} \sum_{i=1}^N \norm{g(h(\widetilde \mZ_i))  - \mZ_i}^2_{\mathrm{F}}, \label{eq:mask pretrain obj}
\end{align}} 
to get representation $\hat{h}$. 

Then, in the fine-tuning stage for a binary classification tasks with labels $y_i \in \{-1,+1\}$, we consider a linear model parameterized by $\btheta$
   \begin{align*}  
   \text{\sffamily{downstream model:}} & \quad f(h( \mX)) = \mathbf{1}^\top \hat{h} (\mX_i)\btheta,  
\end{align*}

with classification loss $\phi(\cdot,\cdot)$ and optimize:
\begin{align*}
   \min_{\norm{\btheta}_2 \leq R} \cR_{\widehat{\cT}}(\btheta\circ \hat{h}(\mX)) = \frac{1}{n}\sum_{i=1}^n \phi(  \mathbf{1}^\top \hat{h} (\mX_i)\btheta,y_i),
\end{align*}
to get $\hat{f}$ (or $\hat\btheta$ in this setting), the aggregated patch over all patches is used for linear projection in classification task.

\noindent\textbf{Roadmap.}~We will provide proof of Theorem~\ref{thm: MAE main} in the following subsections. The roadmap is that in Appendix~\ref{app:proof MAE transfer} we first show the MAE pre-training admits bounded representation transferrability to downstream task (Lemma~\ref{lem:MAE transfer}), and then  in Appendix~\ref{app:proof MAE generalization pre-train} we prove the generalization of MAE pre-training task (Lemma~\ref{lem:MAE gen}). The heart of the proof in this part is to derive worst case covering number of transformer class. 
Finally in Appendix~\ref{app:proof MAE main} we conclude the proof for Theorem~\ref{thm: MAE main} by showing that the representation-induced Rademacher complexity is bounded.

\subsection{Proof of Task Transferability of MAE}\label{app:proof MAE transfer}
Similar to proof of DAE transferability, we define the following quantity:
\begin{align*}
   & \Delta^{ft}_{\cU}(\hat{h}, h^*_{\cU}) = \min_{\norm{\btheta}\leq R} \E_{(\tilde\mZ ,\mZ)\sim \cU }[\phi(\btheta^\top ( \mathbf{1}^\top \hat{h}(\tilde\mZ)   )^\top  )]-\min_{\norm{\widetilde\btheta}\leq   R} \E_{(\tilde\mZ ,\mZ)\sim \cU }[\phi(\tilde{\btheta}^\top ( \mathbf{1}^\top  {h}^*_{\cU}(\tilde\mZ)   )^\top  )],\\
   & \Delta^{pt}_{\cU} (\hat{h}, h^*_{\cU}) = \min_{\mW^D \in \R } \E_{(\tilde\mZ ,\mZ)\sim \cU }\norm{   \hat h(\tilde\mZ)\mW^D  - \mZ  }_{\mathrm{F}}^2   - \E_{(\tilde\mZ ,\mZ)\sim \cU }\norm{ h^*_{\cU}(\tilde\mZ) \mW^{D*} - \mZ }_{\mathrm{F}}^2
\end{align*} 
where $\mathbf{1} = [1,1,1,...]\in\R^K$.

\paragraph{Upper bounding $\Delta^{ft}_{\cU}(\hat{h}, h^*_{\cU})$}
We examine $\Delta^{ft}_{\cU}(\hat{h}, h^*_{\cU})$ first. Similar to DAE proof, We define the optimal head for classification task on distribution $\cU_{\cX}$ under representation $h^*_{\cU}$ as $\tilde\btheta^* = \arg\min_{\norm{\tilde\btheta}\leq  R} \E_{\tilde\mZ  \sim \cU_{\cX}}[\phi(\tilde{\btheta}^\top   (\mathbf{1}^\top   {h}^*_{\cU} (\tilde\mZ )  )   )]$.
\begin{align*}
    &\Delta^{ft}_{\cU}(\hat{h}, h^*_{\cU}) =  \min_{\norm{\btheta}\leq R} \E_{(\tilde\mZ ,\mZ)\sim \cU }[\phi(\btheta^\top (\mathbf{1}^\top \hat{h}(\tilde\mZ) )^\top  )]-  \E_{(\tilde\mZ ,\mZ)\sim \cU }[\phi({\tilde \btheta }^{\star \top}     (\mathbf{1}^\top{h}^*_{\cU} (\tilde\mZ)  )^\top   )]\\
    &\leq  \min_{\norm{\btheta}\leq R} G_\phi \E_{(\tilde\mZ ,\mZ)\sim \cU }| (\btheta^\top  (\mathbf{1}^\top \hat{h}(\tilde\mZ) )^\top   ) -    ({\tilde \btheta }^{\star \top}    (\mathbf{1}^\top h^*_{\cU}(\tilde\mZ) )^\top    )|\\
    &\leq  \min_{\norm{\btheta}\leq R} G_\phi \sqrt{\E_{(\tilde\mZ ,\mZ)\sim \cU }\pare{ \btheta^\top (\mathbf{1}^\top \hat{h}(\tilde\mZ) )^\top   -    {\tilde \btheta }^{\star \top}    (\mathbf{1}^\top {h}^*_{\cU}(\tilde\mZ) )^\top     }^2}\\
    &=  \min_{\norm{\btheta}\leq R} \\
    &\quad  G_\phi\sqrt{   \btheta^\top \E\left[ (\mathbf{1}^\top \hat{h}(\tilde\mZ) )^\top (\mathbf{1}^\top \hat{h}(\tilde\mZ) )\right] \btheta   - 2 \btheta^\top  \E\left[(\mathbf{1}^\top \hat{h}(\tilde\mZ) )^\top (\mathbf{1}^\top  {h}^*_{\cU}(\tilde\mZ) ) \right]\tilde{\btheta}^\top       +  {\tilde \btheta }^{\star \top} \E\left[(\mathbf{1}^\top  {h}^*_{\cU}(\tilde\mZ) )^\top(\mathbf{1}^\top \hat{h}^*_{\cU}(\tilde\mZ) ) \right] \tilde{\btheta}  }
\end{align*}
Since $\sqrt{f(x)}$ and $f(x)$ attain the minimum at the same point, we examine the minimum of above statement over $ \btheta$ without square root. Under unconstrained setting, the minimum of above statement is $  {\tilde \btheta }^{\star \top} \Lambda\pare{\mathbf{1}^{\top} \hat{h}(\tilde \mZ),\mathbf{1}^{ \top} h^*_{\cU}(\tilde \mZ)}   \tilde\btheta^*$ when $ \btheta^* = \pare{\E\left[ (\mathbf{1}^\top \hat{h}(\mX) )^\top (\mathbf{1}^\top \hat{h}(\mX) )\right]}^{\dagger}\E\left[(\mathbf{1}^\top \hat{h}(\mX) )^\top (\mathbf{1}^\top  {h}^*_{\cU}(\mX) )\right] \tilde\btheta^*$.
Hence we have 
\begin{align*}
    \Delta^{ft}_{\cU}(\hat{h}, h^*_{\cU})  &\leq    \sqrt{{\tilde \btheta }^{\star \top} \Lambda\pare{\mathbf{1}^{\top} \hat{h}(\tilde \mZ),\mathbf{1}^{ \top} h^*_{\cU}(\tilde \mZ)} \tilde\btheta^*  } = \sqrt{\trace(\Lambda\pare{\mathbf{1}^{\top} \hat{h}(\tilde \mZ),\mathbf{1}^{ \top} h^*_{\cU}(\tilde \mZ)}\tilde\btheta^*{\tilde \btheta }^{\star \top} )}\\
    &\leq  \sqrt{ d\sigma_{\max}(\Lambda\pare{\mathbf{1}^{\top} \hat{h}(\tilde \mZ),\mathbf{1}^{ \top} h^*_{\cU}(\tilde \mZ)}) \sigma_{\max}(\tilde\btheta^*{\tilde \btheta }^{\star \top}) },
\end{align*}
where 
\begin{align*}
    &\Lambda\pare{\mathbf{1}^{\top} \hat{h}(\tilde \mZ),\mathbf{1}^{ \top} h^*_{\cU}(\tilde \mZ)} \\
    &=  \E\left[(\mathbf{1}^\top  {h}^*_{\cU}(\tilde \mZ) )^\top(\mathbf{1}^\top  {h}^*_{\cU}(\tilde \mZ) )\right] -\E\left[(\mathbf{1}^\top \hat{h}(\tilde \mZ) )^\top (\mathbf{1}^\top  {h}^*_{\cU}(\tilde \mZ) ) \right]\pare{\E\left[(\mathbf{1}^\top \hat{h}(\tilde \mZ) )^\top (\mathbf{1}^\top \hat{h}(\tilde \mZ) )\right] }^{\dagger}\E\left[(\mathbf{1}^\top \hat{h}(\tilde \mZ) )^\top (\mathbf{1}^\top  {h}^*_{\cU}(\tilde \mZ) )  \right] . 
\end{align*}

At last, by choosing a properly large $R$, we can guarantee the optimum can be attained.

\paragraph{Lower bounding $\Delta^{pt}_{\cU}(\hat{h}, h^*_{\cU})$}
 
Similar to CE proof, we have:

\begin{align*}
    & \Delta^{pt}_{\cU} (\hat{h}, h^*_{\cU})\\
   &= \min_{\mW^D \in \R^{d\times d} } \E_{(\tilde\mZ ,\mZ)\sim \cU }\norm{ \hat h(\widetilde \mZ) \mW^D  -\mZ }_{\mathrm{F}}^2  - \E_{(\tilde\mZ ,\mZ)\sim \cU }\norm{   h^*_{\cU}(\widetilde \mZ) {\mW^D }^*  -\mZ  }_{\mathrm{F}}^2\\ 
    &= \min_{\mW^D \in \R^{d\times d} } \E_{(\tilde\mZ ,\mZ)\sim \cU }\norm{ \hat h(\widetilde \mZ) \mW^D   -   h^*_{\cU}(\widetilde \mZ) {\mW^D }^*   }_{\mathrm{F}}^2\\  
    & \geq \sum_{i=1}^d \min_{\bw_r  \in \R^d }\E_{(\tilde\mZ ,\mZ)\sim \cU }\norm{   \hat h(\widetilde \mZ) \bw_r  -h^*_{\cU}(\widetilde \mZ) \bw_r^*}^2\\
   & = \sum_{i=1}^d \min_{\bw_r  \in \R^d }\E_{(\tilde\mZ ,\mZ)\sim \cU } \pare{\bw_r^\top \hat h(\widetilde \mZ)^\top  \hat h(\widetilde \mZ) \bw_r  -2\bw_r^\top \hat h(\widetilde \mZ)^\top h^*_{\cU}(\widetilde \mZ) \bw_r^* + {\bw_r^*}^\top h^*_{\cU}(\widetilde \mZ)^\top h^*_{\cU}(\widetilde \mZ) \bw_r^*}
\end{align*}

where the second step is due to our realizability Assumption~\ref{assumption: realizable},  $\bw_r \in \R^d$ represents $r$th colum of $\mW^D$ and so is $\bw_r^*\in \R^d$ represents $r$th colum of ${\mW^D }^*$.

 
Similar to Context Encoder proof, we define Schur complement as:
 
\begin{align*}
    &\Lambda\pare{ \hat{h}(\tilde \mZ), h^*_{\cU}(\tilde \mZ)} = \E\left[( h^*_{\cU}( \tilde\mZ) )^\top h^*_{\cU}( \tilde \mZ)\right]   -\E\left[({h}^*_{\cU}(\tilde\mZ))^\top (\hat h( \tilde\mZ) ) \right]\pare{\E\left[( h^*_{\cU}( \tilde\mZ) )^\top{{h}^*_{\cU}}^\top(\tilde\mZ) \right] }^{\dagger}\E\left[(\hat h( \tilde\mZ) )^\top{{h}^*_{\cU}}(\tilde\mZ) \right] . 
\end{align*}
 By computing the closed form solution of quadratic form we arrived at:
\begin{align*}
       \Delta^{pt}_{\cU} (\hat{h}, h^*_{\cU})    & \geq       \sum_{r=1}^d   \bw_r^{*\top} \Lambda\pare{  \hat{h} (\tilde \mZ),  h^* (\tilde \mZ)} \bw_r^*   \\ 
     & \geq  \textrm{tr}\pare{ \Lambda\left(  \hat{h} (\tilde \mZ), h^* (\tilde \mZ)\right)    \sum_{j=1}^d\bw_j^* \bw_j^{*\top}}  \\
     &\geq   \sigma_{\max}\pare{\Lambda\pare{ \hat{h} (\tilde \mZ),  h^* (\tilde \mZ)} } \sigma_{\min}\pare{ \sum_{j=1}^d\bw_j^*\bw_j^{*\top}} . 
\end{align*}
 Recall that 
 \begin{align*}
     \Delta^{ft}_{\cU} (\hat{h}, h^*_{\cU}) \leq \sqrt{d \sigma_{\max}\pare{\Lambda\pare{\mathbf{1}^{\top} \hat{h}(\tilde \mZ),\mathbf{1}^{ \top} h^*_{\cU}(\tilde \mZ)}} \sigma_{\max}(\tilde\btheta^*{\tilde \btheta }^{\star \top}) }.
 \end{align*}
Hence, we can conclude that
\begin{align*}
   \frac{\Delta^{ft}_{\cU} (\hat{h}, h^*_{\cU}) }{\pare{\Delta^{pt}_{\cU} (\hat{h}, h^*_{\cU}) }^{1/2}    }  \leq O\pare{\frac{ \sqrt{  \sigma_{\max}\pare{\Lambda\pare{\mathbf{1}^{\top} \hat{h}(\tilde \mZ),\mathbf{1}^{ \top} h^*_{\cU}(\tilde \mZ)}}  }}{\sqrt{  \sigma_{\max}\pare{\Lambda\pare{  \hat{h} (\tilde \mZ),  h^*_{\cU} (\tilde \mZ)} } } } \sqrt{\frac{d\sigma_{\max}(\tilde\btheta^*{\tilde \btheta }^{\star \top})}{\sigma_{\min}\pare{ \sum_{j=1}^d\bw_j^*\bw_j^{*\top}}}}},
\end{align*}
which indicates that MAE pre-training admits an
$$\pare{\Omega\pare{\frac{ \sqrt{  \sigma_{\max}\pare{\Lambda\pare{\mathbf{1}^{\top} \hat{h}(\tilde \mZ),\mathbf{1}^{ \top} h^*_{\cU}(\tilde \mZ)}}  }}{\sqrt{  \sigma_{\max}\pare{\Lambda\pare{  \hat{h} (\tilde \mZ),  h^*_{\cU} (\tilde \mZ)} } } } \sqrt{\frac{d\sigma_{\max}(\tilde\btheta^*{\tilde \btheta }^{\star \top})}{\sigma_{\min}\pare{ \sum_{j=1}^d\bw_j^*\bw_j^{*\top}}}}},\frac{1}{2}}$$

representation transferrability to binary classification task. Notice that the transfer constant $C_{\beta}$ mainly depends on the Schur complement of $\hat{h} (\tilde \mZ),  h^*_{\cU} (\tilde \mZ)$, and $\mathbf{1}^{\top} \hat{h}(\tilde \mZ),\mathbf{1}^{ \top} h^*_{\cU}(\tilde \mZ)$. In the main paper Lemma~\ref{lem:MAE transfer} we omit this constant dependency.

\subsection{Proof of Generalization of MAE Pre-training Task}\label{app:proof MAE generalization pre-train}
In this section we are going to derive generalization bound of the masking pre-training with Transformer. 
In pursuit of optimal generalization bound of pretraining task, i.e., regression with deep transformer, we again need to employ the framework we introduced in CE analysis (Appendix~\ref{app:proof generalization DAE}). Hence, we need to upper bound the worst case $L_2$ covering number of deep transformer class.  The following result establishes the worst case spectral covering number of $L$-layer self-attention transformer defined in~\ref{eq:transformer class}.

\begin{lemma}[Covering number of transformer class]\label{lem: transformer cover}
Let $\mX_{[N]} = [\mX_1;...\mX_N] \in \R^{NK\times d}$ denotes the concatenated data matrix. Then the worst case covering number of $L$-layer transformer class $\cH$ defined in~\ref{eq:transformer class} is bounded as follows:
\begin{align*}
    \ln \cN_\infty(\cH(\cS), \epsilon, \norm{\cdot} ) \leq    O\pare{ s^2_{L}\norm{\mX_{[N]}}^2 \sum_{l=1}^L \frac{\rho_l  }{\epsilon^2}},
\end{align*}
where  
\begin{align*}
s_l &:= \prod_{j=1}^l \pare{\alpha_2  W^2(j) +1} \pare{ W^2(j)\alpha_1  K   + 1  },\\
  \rho_{l}&:=   O\pare{ {\alpha_1^2(\alpha_2W^2(l) +  1)^2  B^2(l) } \ln(2d^2) \pare{K^2 + \frac{\alpha_1W^2(l)\pare{s_{l-1}\norm{\mX_* }  }^2}{d_K}  }  } \\
        &\quad+ O\pare{ {\alpha_2^2W^2(l) B^2(l) (W^2(l) + \alpha_1^2 K^2W^2(l) )}   \ln (2dm)},\\
       \norm{\mX_*}  &:= \max_{i\in[N]}\norm{\mX_i}.
\end{align*}
\end{lemma}
Roughly speaking, $\rho_l$ is the price for covering the parameter of $l$th self-attention layer, and extending the cover to the whole model yields the sum over $l$.  Notice that to ensure a $L_{\infty}$ cover, it suffices to ensure that $ {\sum_{i=1}^N \norm{h(\mX_i) - h_{\epsilon}(\mX_i)}^2} \leq \epsilon^2$. However, if we trivially cover each individual loss $\norm{h(\mX_i) - h_{\epsilon}(\mX_i)}^2$ with $\epsilon^2/N$ radius, the final covering number will be $N$ times larger, which make the later generalization bound vacuous, i.e., greater than 1. To avoid this $N$ factor, we directly consider the cover over concatenated data matrix $\mX_{[N]}$, and consider the covering $\norm{\hat h(\mX_{[N]}) - h(\mX_{[N]})}^2 \leq \epsilon^2 $. Using the fact that matrix covering bound is independent of dimension of $\mX_{[N]}$, but only depends the spectral norm of $\mX_{[N]}$, the final covering number will only have logrithmic dependency on $N$.

To prove Lemma~\ref{lem: transformer cover}, first we introduce the following matrix covering number bound from~\citep{bartlett2017spectrally}.
\begin{lemma}{\cite[Lemma 3.2]{bartlett2017spectrally}}
  \label{fact:matrix_l21_covering}
  Let conjugate exponents $(p,q)$ and $(r,s)$ be given with
  $p \leq 2$,
  as well as positive reals $(a,b,\epsilon)$ and positive integer $m$.
      Let matrix $\mX \in \R^{NK \times d}$ be given with $\|\mX\|_{p} \leq b$.
  Then
    \[
    \ln \cN\pare{\cbr{\mX\mW : \mW\in \R^{d\times m}, \|\mW\|_{q,s}\leq a}, \epsilon, \|\cdot\|_2}
    \leq
    \left\lceil \frac{a^2 b^2 m^{2/r}}{ \epsilon^2}\right \rceil \ln(2dm).
  \]
\end{lemma}

\begin{lemma}[Covering number of attention matrix] \label{lem: softmax covering}
Given a set of data $\cS = \cbr{\mX_1,...,\mX_N}$ and the attention matrix class:
$$\cH_S (\cS) = \cbr{
         \begin{aligned}
        \mS= \begin{bmatrix}
            \mS_1, \mathbf{0},...,\mathbf{0},\\
            \ddots\\
            \mathbf{0},...,\mathbf{0}, \mS_N
        \end{bmatrix}: \mS_i = \mathrm{softmax}\pare{\frac{1}{\sqrt{d_K}} \mX_i \mW_K (\mX_i\mW_Q)^\top }  :
         &  \norm{\mW_{K}},\norm{\mW_{Q}} \leq W,\\ & \norm{\mW_{K}}_{2,1},\norm{\mW_{Q}}_{2,1} \leq B,\\
          \end{aligned}
         }$$
the following covering number bound holds true:
    \begin{align*}
        \ln\cN(\cH_S (\cS), \epsilon, \norm{\cdot}) \leq   O\pare{ \frac{KW^2 B^2\norm{\mX_*}^4}{d_K\epsilon^2}   \ln (2d^2)}.
    \end{align*}
\begin{proof}
     we define set $\cK = \cbr{ \mX \mW_K: \norm{\mW_K} \leq W, \norm{\mW_K}_{2,1} \leq B }$, $\cQ = \cbr{ \mX \mW_Q: \norm{\mW_Q} \leq W, \norm{\mW_Q}_{2,1} \leq B }$. We define $\epsilon_K$ cover of $\cK$ as $\cC_K$, and $\epsilon_Q$ cover of $\cQ$ as $\cC_Q$. We construct the following set:
     \begin{align*}
         \cC_S =  \cbr{
         \begin{aligned}
        \mS= \begin{bmatrix}
            \mS_1, \mathbf{0},...,\mathbf{0},\\
            \ddots\\
            \mathbf{0},...,\mathbf{0}, \mS_N
        \end{bmatrix}: \mS_i = \mathrm{softmax}\pare{\frac{1}{\sqrt{d_K}} \mX_i \mW_K (\mX_i\mW_Q)^\top }  :
        \mW_K \in \cC_K, \mW_Q \in \cC_Q
          \end{aligned}
         }
     \end{align*}

     Next we will show that $\cC_S$ is a cover of $\cH_S(\cS)$ with some radius. For any $\mS_{[N]} \in \cH_S$, we can find $\hat{\mS}_{[N]} \in \cC_S$ such that:
     \begin{align*}
         \norm{\mS_{[N]} - \hat{\mS}_{[N]}} &\leq \max_{i\in[N]}  \norm{\mS_{i} - \hat{\mS}_{i}}\\
         &= \max_{i\in[N]}  \norm{\mathrm{softmax}\pare{\frac{1}{\sqrt{d_K}} \mX_i \mW_K (\mX_i\mW_Q)^\top }  - \mathrm{softmax}\pare{\frac{1}{\sqrt{d_K}} \mX_i \hat\mW_K (\mX_i\hat\mW_Q)^\top } }\\
          &\leq \frac{\sqrt{K}}{\sqrt{d_K}} \max_{i\in[N]}  \norm{   \mX_i \mW_K (\mX_i\mW_Q)^\top    -   \mX_i \hat\mW_K (\mX_i\hat\mW_Q)^\top   }\\
          &\leq \frac{\sqrt{K}}{\sqrt{d_K}} \max_{i\in[N]}  \norm{ (  \mX_i \mW_K    -   \mX_i \hat\mW_K )(\mX_i\mW_Q)^\top   }\\
          &\quad + \frac{\sqrt{K}}{\sqrt{d_K}} \max_{i\in[N]}  \norm{   \mX_i \hat\mW_K  (\mX_i \mW_Q)^\top  -   \mX_i \hat\mW_K (\mX_i\hat\mW_Q)^\top   }\\
          & \leq \frac{W\sqrt{K}}{\sqrt{d_K}} \pare{\epsilon_K +\epsilon_Q} \max_{i\in[N]}\norm{\mX_i},
     \end{align*}
     where the first inequality is due to the property of block diagonal matrices. We define $ \norm{\mX_*} = \max_{i\in[N]}\norm{\mX_i}$.
     To ensure above bound is less than $\epsilon$, we  choose $\epsilon_K = \epsilon_Q =  \frac{\sqrt{d_K}}{2W\sqrt{K} \norm{\mX_*}}\epsilon$. According to Lemma~\ref{fact:matrix_l21_covering}, we know:
     \begin{align*}
         \ln |\cC_S| \leq \ln |\cC_K| + \ln |\cC_Q| \leq O\pare{ \frac{KW^2 B^2\norm{\mX_*}^4}{d_K\epsilon^2}   \ln (2d^2)}.
     \end{align*}
\end{proof}

\end{lemma}
 
\begin{proposition}[Covering number of single self-attention layer]
    Consider the following function class of self-attention module:
    \begin{align*}
        \cH_{SA} := \cbr{
         \begin{aligned}
        \mX\mapsto \sigma\pare{\mZ\mW_{\mathrm{FC1}}}\mW_{\mathrm{FC2}}+\mZ:  &\mZ = \pare{\mA+\mX}, \mA = \mathrm{softmax}\pare{\frac{1}{\sqrt{d_K}} \mX \mW_K (\mX\mW_Q)^\top } \mX \mW_V\\
         & \norm{\mW_{\mathrm{FC1}}},\norm{\mW_{\mathrm{FC2}}},\norm{\mW_{K}},\norm{\mW_{Q}},\norm{\mW_{V}} \leq W,\\ &\norm{\mW_{\mathrm{FC1}}}_{2,1},\norm{\mW_{\mathrm{FC2}}}_{2,1},\norm{\mW_{K}}_{2,1},\norm{\mW_{Q}}_{2,1},\norm{\mW_{V}}_{2,1} \leq B,\\
          \end{aligned}
         }
    \end{align*} 
    then the following bound holds for its covering number:
    \begin{align*}
        \ln \cN(\cH_{SA}(\cS) ,\epsilon, \norm{\cdot} )&\leq O\pare{ \frac{(\alpha_1\alpha_2W^2 +  \alpha_1)^2  B^2 \norm{\mX_{[N]}}^2 }{ \epsilon^2 }\ln(2d^2)} \pare{K^2 + \frac{\alpha_1W^2\norm{\mX_*}^2}{d_K}  }\\ 
            &\quad + O\pare{\frac{\alpha_2^2W^2 B^2 (W^2\norm{\mX_{[N]}}^2 + \alpha_1^2 K^2W^2\norm{\mX_{[N]}}^2)}{\epsilon^2 }  \ln (2dm)}  .
    \end{align*}
    \begin{proof}
    Recall that $\mX_{[N]} \in \R^{NK\times d}$ is the concatenated data matrix, and we shall use $h(\mX_{[N]}) \in \R^{NK\times d} $ to denote the concatenated encoder output, i.e., $h(\mX_{[N]}) = [h(\mX_1);...,h(\mX_N)]$.
    Our goal is to find the cardinality of a cover such that for any $h\in\cH$ we can find a $h_{\epsilon} \in \cC_{\mathrm{SA}}$ such that
    \begin{align*}
        \norm{h(\mX_{[N]}) - h_\epsilon(\mX_{[N]})} \leq \epsilon.
    \end{align*}
    \item
    \paragraph{I: Covering number of input layer by value matrix}
        Let $\cC_V$ to be $\epsilon_V$ cover of set $\cH_{V}(\cS) = \cbr{ \mX_{[N]}\mW_V: \norm{\mW_V} \leq W, \norm{\mW_V}_{2,1} \leq B }$, then evoking Lemma~\ref{fact:matrix_l21_covering} we have:
        \begin{align*}
            \ln \cN(\cH_V,\epsilon_V,\norm{\cdot}) \leq O\pare{ \frac{B^2 \norm{\mX_{[N]}}^2 }{ \epsilon^2_V}  \ln(2dm)}.
        \end{align*}
    \item
    \paragraph{II: Covering number of Attention layer}
        Next, consider the set of attention matrix
        $$\cH_S (\cS) = \cbr{
         \begin{aligned}
        \mS= \begin{bmatrix}
            \mS_1, \mathbf{0},...,\mathbf{0},\\
            \ddots\\
            \mathbf{0},...,\mathbf{0}, \mS_N
        \end{bmatrix}: \mS_i = \mathrm{softmax}\pare{\frac{1}{\sqrt{d_K}} \mX_i \mW_K (\mX_i\mW_Q)^\top }  :
         &  \norm{\mW_{K}},\norm{\mW_{Q}} \leq W,\\ & \norm{\mW_{K}}_{2,1},\norm{\mW_{Q}}_{2,1} \leq B,\\
          \end{aligned}
         }$$
        
         From Lemma~\ref{lem: softmax covering} we know its covering number can be bounded as:
        \begin{align*}
            \ln \cN \pare{\cH_S(\cS), \epsilon, \norm{\cdot}} \leq \ln \cN \pare{\cH_{\tilde{S}}, \epsilon_S, \norm{\cdot}}  \leq    O\pare{ \frac{KW^2 B^2\norm{\mX_*}^4}{d_K\epsilon^2_S}   \ln (d^2)}.
        \end{align*}
        Now we can proceed to bounding the covering number of following set:
        
        $$\cH_A(\cS) = \cbr{
         \begin{aligned}    \alpha_1\mathrm{softmax}\pare{\frac{1}{\sqrt{d_K}} \mX_{[N]} \mW_K (\mX_{[N]}\mW_Q)^\top } \mX_{[N]}\mW_V :
         &  \norm{\mW_{K}},\norm{\mW_{Q}},\norm{\mW_{V}} \leq W,\\ & \norm{\mW_{K}}_{2,1},\norm{\mW_{Q}}_{2,1},\norm{\mW_{V}}_{2,1} \leq B,\\
          \end{aligned}
         }$$

         For every element $\hat\mV_{[N]}\in \cC_V$, we construct the set 
         $ \alpha_1\cH_S(\cS)\circ \hat\mV_{[N]}: = \cbr{\alpha_1\mS_{[N]} \hat\mV_{[N]}: \mS_{[N]} \in \cH_S(\cS)}$. Then we define $ \epsilon_A$-covering of $\cH_S\circ \hat\mV$ as $\cC(\cH_S\circ \hat\mV, \epsilon_A, \norm{\cdot})$. To construct $\cH_S\circ \hat\mV$ as $\cC(\cH_S\circ \hat\mV, \epsilon_A, \norm{\cdot})$, we consider $\cC_S$.
         For any $\mS_{[N]} \hat\mV_{[N]} \in \cH_S(\cS)\circ \hat\mV_{[N]}$, we can find $\hat\mS_{[N]} \in \cC_S$, such that
         \begin{align*}
             \norm{\alpha_1\mS_{[N]} \hat\mV_{[N]} - \alpha_1\hat\mS_{[N]} \hat\mV_{[N]}} &\leq \alpha_1\norm{\mS_{[N]}  - \hat\mS_{[N]}}\norm{\hat\mV_{[N]}}\\
             &\leq \alpha_1\epsilon_S \norm{\mX_{[N]}}W.
         \end{align*}
         Setting $\epsilon_S = \frac{\epsilon_A}{\alpha_1\norm{\mX_{[N]}}W}$ we can conclude that $\cC(\cH_S\circ \hat\mV, \epsilon_A, \norm{\cdot})$ actually $\epsilon_A$ covers $\cH_S\circ \hat\mV$ and the following fact holds for the covering  number
         \begin{align*}
          \ln  |\cC( \cH_S(\cS)\circ \hat\mV_{[N]},  \epsilon_A, \norm{\cdot})| \leq \sup_{\hat\mV_{[N]} \in \cC_V}  \ln \cN(  \cH_{S}(\cS)\circ \hat\mV_{[N]},   \epsilon_A, \norm{\cdot})  \leq O\pare{ \frac{\alpha_1^2 K B^2W^4\norm{\mX_*}^4 \norm{\mX_{[N]}}^2}{d_K\epsilon^2_A}   \ln (2d^2)}.
         \end{align*}
         
         Then we construct a cover $\cC_A$ for $\cH_A$ by:
         \begin{align*}
             \cC_A = \bigcup_{\hat\mV_{[N]} \in \cC_V} \cC(\alpha_1 \cH_S(\cS) \circ \hat\mV_{[N]})
         \end{align*}
         It is not hard to verify the cardinality of this cover:
         \begin{align*}
            \ln |\cC_A| &\leq \ln |\cC_V| + \sup_{\hat\mV_{[N]} \in \cC_V} \ln  |\cC( \alpha_1\cH_S(\cS)\circ \hat\mV_{[N]},  \epsilon_A, \norm{\cdot})| \\
           & \leq O\pare{ \frac{B^2 \norm{\mX_{[N]}}^2 }{ \epsilon^2_V}\ln(2dm)} +O\pare{ \frac{\alpha_1^2 K B^2W^4\norm{\mX_*}^4 \norm{\mX_{[N]}}^2}{d_K\epsilon^2_A}   \ln (2d^2)}.
         \end{align*}

        \item
        \paragraph{III: Covering number of fully-connected layer 1}
        By similar reasoning, we can show that the covering number of $$\cH_{\mathrm{FC1}}(\cS) = \cbr{ \mZ_{[N]}\mW_{\mathrm{FC1}}: \mZ_{[N]} = \alpha_1\mA_{[N]}+\mX_{[N]}, \mA_{[N]} \in \cH_A(\cS), \norm{\mW_{\mathrm{FC1}}} \leq W, \norm{\mW_{\mathrm{FC1}}}_{2,1} \leq B  }$$
        
        For every element $\hat \mA_{[N]}  \in \cC_A$, we define set 
        
        $$  \hat\mA_{[N]} \circ \cW_{\mathrm{FC1}} = \cbr{(\alpha_1\hat\mA_{[N]}+\mX_{[N]})\mW_{\mathrm{FC1}}, \norm{\mW_{\mathrm{FC1}}} \leq W, \norm{\mW_{\mathrm{FC1}}}_{2,1} \leq B}$$

        We denote $\epsilon_{\mathrm{FC1}}$-cover of $\hat\mA_{[N]}\circ \cW_{\mathrm{FC1}}$ as $\cC(\hat\mA_{[N]}\circ \cW_{\mathrm FC_1}, \epsilon_{\mathrm{FC1}},\norm{\cdot})$, and the covering number of $\hat\mA_{[N]} \circ \cW_{\mathrm{FC1}}$ is bounded by:
        \begin{align*}
           \ln |\cC(\hat\mA_{[N]}\circ \cW_{\mathrm{FC1}}, \epsilon_{\mathrm{FC1}},\norm{\cdot})| &\leq \sup_{\hat\mA_{[N]} \in \cC_A} \ln \cN(\hat\mA_{[N]}\circ \cW_{\mathrm{FC1}},\epsilon_{\mathrm{FC1}},\norm{\cdot}) \\
            &= \sup_{\hat\mS_{[N]}\in \cC_S}   O\pare{\frac{B^2 (\norm{\mX_{[N]}}^2 + \alpha_1^2W^2\norm{\mX_{[N]}}^2\norm{\hat\mS_{[N]}}^2)}{\epsilon_{\mathrm{FC1}}^2} \ln (dm)}\\
            & \leq   O\pare{\frac{B^2 ( \norm{\mX_{[N]}}^2 + \alpha_1^2 K^2W^2\norm{\mX_{[N]}}^2)}{\epsilon_{\mathrm{FC1}}^2} \ln (dm)}\\
        \end{align*}
        Now, we construct the $\epsilon_{\mathrm{FC1}}$-cover of $\cH_{\mathrm{FC1}}(\cS)$ as
        \begin{align*}
            \cC_{\mathrm{FC1}} = \bigcup_{\hat\mA_{[N]} \in \cC_A} \cC(\hat\mA_{[N]}\circ \cW_{\mathrm{FC1}}, \epsilon_{\mathrm{FC1}},\norm{\cdot})
        \end{align*}
        And the covering number is bounded:
        \begin{align*}
           \ln|\cC_{\mathrm{FC1}} | &\leq \ln |\cC_A| + \sup_{\hat\mA_{[N]}\in\cC_A}\ln |\cC(\hat\mA_{[N]}\circ \cW_{\mathrm{FC1}},  \epsilon_{\mathrm{FC1}},\norm{\cdot})| \\
            &\leq O\pare{ \frac{B^2 \norm{\mX_{[N]}}^2 }{ \epsilon^2_V}\ln(2dm)} +O\pare{ \frac{\alpha_1^2 K B^2W^4\norm{\mX_*}^4 \norm{\mX_{[N]}}^2}{d_K\epsilon^2_A}   \ln (2d^2)} \\
            &\quad +   O\pare{\frac{B^2 ( \norm{\mX_{[N]}}^2 + \alpha_1^2 K^2W^2\norm{\mX_{[N]}}^2)}{\epsilon_{\mathrm{FC1}}^2} \ln (dm)}.
        \end{align*}
        
       \item 
       \paragraph{IV: Covering number of fully-connected layer 2} The analysis this part is almost identical to \textbf{III}. We try to find the covering number of the set $\cH_{\mathrm{SA}}$.  For every element $\hat \mF_{[N]} \in \cC_{\mathrm{FC1}}$ and $\hat \mA_{[N]} \in \cC_A$, define the set
       $$  \alpha_2\hat\mF_{[N]} \circ \cW_{\mathrm{FC2}} + \hat\mZ_{[N]} = \cbr{\alpha_2\sigma(\hat{\mF}_{[N]})\mW_{\mathrm{FC2}} +\hat\mZ_{[N]}:\hat\mZ_{[N]}= \alpha_1\hat\mA_{[N]}+\mX_{[N]}, \norm{\mW_{\mathrm{FC2}}} \leq W, \norm{\mW_{\mathrm{FC2}}}_{2,1} \leq B}$$.

         We denote $\epsilon_{\mathrm{FC2}}$-cover of $\alpha_2\hat\mF_{[N]} \circ \cW_{\mathrm{FC2}} + \hat\mZ$ as $\cC(\alpha_2 \hat\mF_{[N]} \circ \cW_{\mathrm{FC2}} + \hat\mZ, \epsilon_{\mathrm{FC2}},\norm{\cdot})$, and the cardinality of this set is bounded by:
        \begin{align*}
           \ln |\cC(\alpha_2\hat\mF_{[N]} \circ \cW_{\mathrm{FC2}} + \hat\mZ_{[N]}, \epsilon_{\mathrm{FC2}},\norm{\cdot})| &\leq \sup_{\hat\mF_{[N]} \in \cC_{\mathrm{FC1}}, \hat\mA \in \cC_A } \ln \cN(\hat\mF_{[N]} \circ \cW_{\mathrm{FC2}} + \hat\mZ_{[N]}, \epsilon_{\mathrm{FC2}},\norm{\cdot}) \\
            &=    O\pare{\frac{\alpha_2^2B^2 \pare{ \norm{\mX_{[N]}}^2 + \alpha_1^2 K^2W^2\norm{\mX_{[N]}}^2 }W^2}{\epsilon_{\mathrm{FC2}}^2} \ln (2dm)}
        \end{align*}
        Now, we construct the $\epsilon_{\mathrm{FC2}}$-cover of $\cH_{\mathrm{FC2}}$ as
        \begin{align*}
            \cC_{\mathrm{FC2}} = \bigcup_{\hat\mF_{[N]} \in \cH_{\mathrm FC1}, \hat{\mA} \in \cH_A} \cC(\hat\mF_{[N]}\circ \cW_{\mathrm{FC2}} + \hat{\mZ}_{[N]}, \epsilon_{\mathrm{FC2}},\norm{\cdot})
        \end{align*}
        And the covering number is bounded:
        \begin{align*}
           \ln|\cC_{\mathrm{FC2}} | &\leq \ln |\cC_A| +  \ln|\cC_{\mathrm FC1}| + \max_{\hat\mA\in\cC_A}\ln |\cC(\hat\mA_{[N]}\circ \cW_{\mathrm{FC1}}, \tilde\epsilon,\norm{\cdot})| \\
            &\leq O\pare{ \frac{B^2 \norm{\mX_{[N]}}^2 }{ \epsilon^2_V}\ln(2dm)} +O\pare{ \frac{\alpha_1^2 K B^2W^4\norm{\mX_*}^4 \norm{\mX_{[N]}}^2}{d_K\epsilon^2_A}   \ln (2d^2)} \\
            &\quad +   O\pare{\frac{B^2 ( \norm{\mX_{[N]}}^2 + \alpha_1^2 K^2W^2\norm{\mX_{[N]}}^2)}{\epsilon_{\mathrm{FC1}}^2} \ln (dm)}\\
            &\quad + O\pare{\frac{\alpha_2^2B^2 \pare{ \norm{\mX_{[N]}}^2 + \alpha_1^2 K^2W^2\norm{\mX_{[N]}}^2 }W^2}{\epsilon_{\mathrm{FC2}}^2} \ln (2dm)}.
        \end{align*}

        \item 
        \paragraph{V: Verification of $\cC_{\mathrm {FC2}}$ being an $\epsilon$ cover of $\cH_{\mathrm{SA}}$}
        It remains to verify $\cC_{\mathrm FC2}$ is an $\epsilon$ cover of $\cH_{\mathrm SA}$. Given any $\mH_{[N]} \in \cH_{\mathrm SA}$, we can find a $\hat \mH_{[N]}  \in \cC_{\mathrm{FC2}}$ such that
      
        \begin{align*}
            \norm{\mH_{[N]}- \hat \mH_{[N]}} 
            & = \norm{\alpha_2\sigma(\mZ_{[N]}\mW_{\mathrm{FC1}})\mW_{\mathrm{FC2}} +\mZ_{[N]} -  \alpha_2\sigma(\hat{\mZ}_{[N]}\hat{\mW}_{\mathrm{FC1}})\hat{\mW}_{\mathrm{FC2}}- \hat{\mZ}_{[N]}} \\
            & \leq \alpha_2\norm{ \sigma(\mZ_{[N]}\mW_{\mathrm{FC1}})\mW_{\mathrm{FC2}}-   \sigma(\hat{\mZ}_{[N]}\hat{\mW}_{\mathrm{FC1}}) {\mW}_{\mathrm{FC2}}}\\
            & \quad + \norm{\alpha_2\sigma(\hat{\mZ}_{[N]}\hat{\mW}_{\mathrm{FC1}}) {\mW}_{\mathrm{FC2}} + \mZ_{[N]}-\alpha_2\sigma(\hat{\mZ}_{[N]}\hat{\mW}_{\mathrm{FC1}})\hat{\mW}_{\mathrm{FC2}} - \hat{\mZ}_{[N]}}  \\
            &\leq \alpha_2 W \norm{\sigma(\mZ_{[N]}\mW_{\mathrm{FC1}}) -  \sigma(\hat{\mZ}_{[N]}\hat{\mW}_{\mathrm{FC1}})  } + \epsilon_{\mathrm{FC2}}+ \norm{\mZ_{[N]} - \hat{\mZ}_{[N]}} .
        \end{align*}
        We bound $\norm{\sigma(\mZ_{[N]}\mW_{\mathrm{FC1}}) -  \sigma(\hat{\mZ}_{[N]}\hat{\mW}_{\mathrm{FC1}})  }$ first as follows:
        \begin{align*}
           \norm{\sigma(\mZ_{[N]}\mW_{\mathrm{FC1}}) -  \sigma(\hat{\mZ}_{[N]}\hat{\mW}_{\mathrm{FC1}})  }  &\leq \norm{(\alpha_1\mA_{[N]} +\mX_{[N]}) \mW_{\mathrm{FC1}} -  (\alpha_1\hat{\mA}_{[N]} +\mX_{[N]})\hat{\mW}_{\mathrm{FC1}}}  \\
            &\leq \norm{(\alpha_1\mA_{[N]} +\mX_{[N]}) \mW_{\mathrm{FC1}} -  (\alpha_1\hat{\mA}_{[N]} +\mX_{[N]}) {\mW}_{\mathrm{FC1}}} \\
            &\quad + \norm{(\alpha_1\hat{\mA}_{[N]} +\mX_{[N]}) \mW_{\mathrm{FC1}} -  (\alpha_1\hat{\mA}_{[N]} +\mX_{[N]})\hat{\mW}_{\mathrm{FC1}}}\\
            &\leq \alpha_1 W\norm{ \mA_{[N]}  - \hat{\mA}_{[N]} } + \epsilon_{\mathrm{FC1}}.
        \end{align*}
        For $\norm{ \mA_{[N]}  - \hat{\mA}_{[N]} }$, we have
        \begin{align*}
            \norm{ \mA_{[N]}  - \hat{\mA}_{[N]} } &= \norm{\mS_{[N]} \mX_{[N]} \mW_{V} - \hat{\mS}_{[N]} \mX_{[N]} \hat{\mW}_{V}} \\
            &\leq \norm{\mS_{[N]} \mX_{[N]} \mW_{V} - {\mS}_{[N]} \mX_{[N]}  \hat{\mW}_{V}}  + \norm{ \mS_{[N]} \mX_{[N]} \hat{\mW}_{V} - \hat{\mS}_{[N]} \mX_{[N]} \hat{\mW}_{V}} \\
            & \leq K  \norm{  \mX_{[N]} \mW_{V} -  \mX_{[N]}  \hat{\mW}_{V}}  + \epsilon_A \\
            & \leq K\epsilon_V + \epsilon_A.
        \end{align*}
        Putting pieces together yields:
\begin{align*}
    \norm{\sigma(\mZ\mW_{\mathrm{FC1}}) -  \sigma(\hat{\mZ}\hat{\mW}_{\mathrm{FC1}})  }  \leq \alpha_1 W(K\epsilon_V + \epsilon_A) + \epsilon_{\mathrm{FC1}}.
\end{align*}
Now we switch to bounding $\norm{\mZ_{[N]} - \hat\mZ_{[N]}}$:
\begin{align*}
    \norm{\mZ_{[N]} - \hat\mZ_{[N]}} =\alpha_1 \norm{\mA_{[N]} - \hat\mA_{[N]}} \leq \alpha_1(K\epsilon_V + \epsilon_A)
\end{align*}
Hence we know:
\begin{align*}
    \norm{\mH_{[N]} - \hat\mH_{[N]}} &\leq \alpha_2 W \pare{\alpha_1W(K\epsilon_V + \epsilon_A)+\epsilon_{\mathrm{FC}_1}} + \alpha_1(K\epsilon_V + \epsilon_A)+ \epsilon_{\mathrm{FC2}}\\
   & = (\alpha_1\alpha_2 W^2K + \alpha_1 K)\epsilon_V + (\alpha_1\alpha_2 W^2+\alpha_1)\epsilon_A +  \alpha_2 W\epsilon_{\mathrm{FC1}} + \epsilon_{\mathrm{FC2}}
\end{align*}
To make sure RHS is less than $\epsilon$, we set
\begin{align*}
    \epsilon_V = \frac{\epsilon}{4(\alpha_1\alpha_2W^2K + \alpha_1 K)}, \epsilon_A = \frac{\epsilon}{4(\alpha_1\alpha_2 W^2 + \alpha_1)}, \epsilon_{\mathrm{FC1}} = \frac{\epsilon}{4 \alpha_2W}, \epsilon_{\mathrm{FC2}} = \frac{\epsilon}{4}.
\end{align*}
Recall that
    \begin{align*}
           \ln|\cC_{\mathrm{FC2}} | &\leq \ln |\cC_A| +  \ln|\cC_{\mathrm FC1}| + \max_{\hat\mA\in\cC_A}\ln |\cC(\hat\mA_{[N]}\circ \cW_{\mathrm{FC1}}, \tilde\epsilon,\norm{\cdot})| \\
            &\leq O\pare{ \frac{B^2 \norm{\mX_{[N]}}^2 }{ \epsilon^2_V}\ln(2dm)} +O\pare{ \frac{\alpha_1^2 K B^2W^4\norm{\mX_*}^4 \norm{\mX_{[N]}}^2}{d_K\epsilon^2_A}   \ln (2d^2)} \\
            &\quad +   O\pare{\frac{B^2 ( \norm{\mX_{[N]}}^2 + \alpha_1^2 K^2W^2\norm{\mX_{[N]}}^2)}{\epsilon_{\mathrm{FC1}}^2} \ln (dm)}\\
            &\quad + O\pare{\frac{\alpha_2^2B^2 \pare{ \norm{\mX_{[N]}}^2 + \alpha_1^2 K^2W^2\norm{\mX_{[N]}}^2 }W^2}{\epsilon_{\mathrm{FC2}}^2} \ln (2dm)}.
        \end{align*}

Hence we can upper bound the covering number of $\cH_{SA}$ as follows:
\begin{align*}
    \ln \cN(\cH_{SA},\epsilon, \norm{\cdot})    &\leq O\pare{ \frac{(\alpha_1\alpha_2W^2 +  \alpha_1)^2K^2 B^2 \norm{\mX_{[N]}}^2 }{ \epsilon^2 }\ln(2dm)}\\
    & \quad+O\pare{ \frac{(\alpha_1\alpha_2 W^2 + \alpha_1)^2\alpha_1^2 K B^2W^4\norm{\mX_*}^4 \norm{\mX_{[N]}}^2}{d_K\epsilon^2 }   \ln (2d^2)} \\
            &\quad + O\pare{\frac{\alpha_2^2W^2 B^2 ( \norm{\mX_{[N]}}^2 + \alpha_1^2 K^2W^2\norm{\mX_{[N]}}^2)}{\epsilon^2 }  \ln (2dm)}  \\ 
    &= O\pare{ \frac{(\alpha_1\alpha_2W^2 +  \alpha_1)^2  B^2 \norm{\mX_{[N]}}^2 }{ \epsilon^2 }\ln(2d^2)} \pare{K^2 + \frac{\alpha_1W^4\norm{\mX_*}^4}{d_K}  }\\ 
            &\quad + O\pare{\frac{\alpha_2^2W^2 B^2 ( \norm{\mX_{[N]}}^2 + \alpha_1^2 K^2W^2\norm{\mX_{[N]}}^2)}{\epsilon^2 }  \ln (2dm)} .
\end{align*}

    \end{proof}
\end{proposition}

\begin{proposition}[Contraction mapping of self-attention layer]
For a single attention layer parameterized by $\mW$, with $\norm{\mW} \leq W$, the following statement holds:
\begin{align*}
        \norm{\mathrm{SA}_{\mW} (\mX) - \mathrm{SA}_{\mW} (\hat\mX)}\leq (\alpha_2 W^2+1)\pare{  \alpha_1 KW  + 1} \norm{  \mX   -   \hat\mX }.
    \end{align*}
    \begin{proof}
    The proof follows by definition and simple algebraic manipulation:
    \begin{align*}
        \norm{\mathrm{SA}_{\mW} (\mX) - \mathrm{SA}_{\mW} (\hat\mX)} &\leq  \norm{\alpha_2\sigma\pare{\mZ\mW_{\mathrm{FC1}}}\mW_{\mathrm{FC2}}+\mZ - \beta\sigma\pare{\hat\mZ\mW_{\mathrm{FC1}}}\mW_{\mathrm{FC2}}-\hat\mZ}\\
       & \leq \alpha_2 W^2\norm{ \mZ   -  \hat\mZ } + \norm{\mZ-\hat\mZ}\\
       & \leq (\alpha_2 W^2+1)\norm{ \mA + \mX   -  \hat\mA - \hat\mX }  \\
       & \leq (\alpha_2 W^2+1)\pare{\norm{ \alpha_1 \mS\mX\mW_V    - \alpha_1\mS \hat\mX \mW_V} + \norm{  \mX   -   \hat\mX }} \\
       & \leq (\alpha_2 W^2+1)\pare{  \alpha_1 KW  + 1} \norm{  \mX   -   \hat\mX }.
    \end{align*}
\end{proof}

\end{proposition}

\subsubsection{Proof of Lemma~\ref{lem: transformer cover}}
    \begin{proof}

     We first examine the norm of each self-attention layer's output: 
    \begin{align}
    \norm{\mX^l_i}  &\leq \alpha_2  W^2(l) \norm{\mZ^{l}_i} + \norm{\mZ^{l}_i} \nonumber\\
    &\leq \pare{\alpha_2  W^2(l) +1} \pare{\alpha \norm{\mA^{l}_i}  + \norm{\mX^{l-1}_i} }\nonumber \\
    &\leq \pare{\alpha_2  W^2(l) +1} \pare{ W^2(l)\alpha  K   + 1  }\norm{\mX^{l-1}}\nonumber \\
     &\leq  \prod_{j=1}^l \pare{\alpha_2  W^2(j) +1} \pare{ W^2(j)\alpha_1  K   + 1  }\norm{\mX_i }   \label{eq: norm of transformer output}
    \end{align}
    and grouped output
    \begin{align*}
    \norm{\mX^l_{[N]}}  &\leq \alpha_2  W^2(l) \norm{\mZ^{l}_{[N]}} + \norm{\mZ^{l}_{[N]}} \\
    &\leq \pare{\alpha_2  W^2(l) +1} \pare{\alpha \norm{\mA^{l}_{[N]}}  + \norm{\mX^{l-1}_{[N]}} } \\
    &\leq \pare{\alpha_2  W^2(l) +1} \pare{ W^2(l)\alpha  K   + 1  }\norm{\mX_{[N]}^{l-1}} \\
     &\leq  \prod_{j=1}^l \pare{\alpha_2  W^2(j) +1} \pare{ W^2(j)\alpha_1  K   + 1  }\norm{\mX_{[N]}} .
    \end{align*}
    For the ease of presentation, we define the class of $l$th layer output:
    \begin{align*}
        \cH_l =  \cbr{
    \begin{aligned} \mathrm{SA}^{l}&\pare{\mathrm{SA}^{l-1}...\mathrm{SA}^{1}(\mX)}: \norm{\mW^j_{\mathrm{FC1}} },\norm{\mW^j_{\mathrm{FC2}} },\norm{\mW^j_{K} },\norm{\mW^l_{Q} },\norm{\mW^j_{V} } \leq W(j),\\ &\norm{\mW^j_{\mathrm{FC1}}}_{2,1},\norm{\mW^j_{\mathrm{FC2}}}_{2,1},\norm{\mW^j_{K}}_{2,1},\norm{\mW^j_{Q}}_{2,1},\norm{\mW^j_{V}}_{2,1} \leq B(j), \forall j \in [l]
    \end{aligned}
    },
    \end{align*}
   and it will be useful to define set of weight matrices at $l$th layer:
    \begin{align*}
        \cW_l =  \cbr{
    \begin{aligned} \mW: \norm{\mW } \leq W(l), \norm{\mW }_{2,1}  \leq B(l) .
    \end{aligned}
    }.
    \end{align*}

    We shall construct the cover with certain radius for each $\cH_l, l \in [L]$. 
    
    For base case $l=1$: we create $\epsilon_1$ cover of $\mathrm{SA}^{1}( \mX_{[N]} )$
    \begin{align*}
        \cC_1 = \cC(\mathrm{SA}^{1}( \mX_{[N]} ),\epsilon_1,\norm{\cdot}   ).
    \end{align*}

    For $1< l+1 \leq L$, for each element $\hat\mX^l \in \cC_l$, we construct the $\epsilon_{l+1}$-cover of the following set:
    \begin{align*}
        \mathrm{SA}^{l+1}(\hat\mX^l):=\cbr{  \mathrm{SA}_{\mW^{l+1}}(\hat\mX^l), \mW^{l+1} \in \cW_{l+1} }
    \end{align*}
    and we denote the cover as $\cC\pare{ \mathrm{SA}^{l+1}(\hat\mX^l), \epsilon_{l+1}, \norm{\cdot}  }$. We first examine the cardinality of this cover as follows:
    \begin{align*}
        &\ln \left|\cC\pare{ \mathrm{SA}^{l+1}(\mX^l_{[N]}), \epsilon_{l+1}, \norm{\cdot}  }\right|\\ 
        &\leq  \max_{\mX^{l} \in \cC_{l}}  O\pare{ \frac{(\alpha_1\alpha_2W^2 +  \alpha_1)^2  B^2 }{ \epsilon^2 }\ln(2d^2)} \pare{K^2 + \frac{\alpha_1W^4\pare{\max_{i\in[N]}\norm{\mX^l_{i}}   }^4}{d_K}  } \norm{\mX^l_{[N]}}^2\\ 
            &\quad + O\pare{\frac{\alpha_2^2W^2(l+1) B^2(l+1) (1 + \alpha_1^2 K^2W^2(l+1) )}{\epsilon^2 }  \ln (2dm)} \norm{\mX^l_{[N]}}^2 \\ 
        & \leq      O\pare{ \frac{(\alpha_1\alpha_2W^2(l+1) +  \alpha_1)^2  B^2(l+1) }{ \epsilon^2 }\ln(2d^2)} \pare{K^2 + \frac{\alpha_1W^4(l+1)\pare{s_l\norm{\mX_* }  }^4}{d_K}  }s_l^2\norm{\mX_{[N]}}^2\\ 
            &\quad + O\pare{\frac{\alpha_2^2W^2(l+1) B^2(l+1) (1 + \alpha_1^2 K^2W^2(l+1) )}{\epsilon^2 }  \ln (2dm)} s_l^2\norm{\mX_{[N]}}^2  \\ 
        &  := \ln N_{l+1}
    \end{align*}
    where $s_l := \prod_{j=1}^l \pare{\alpha_2  W^2(j) +1} \pare{ W^2(j)\alpha_1  K   + 1  }$.
 
    We then construct cover for $\cH_{l+1}$ as:
    \begin{align*}
         \cC_{l+1} = \bigcup_{\mX^l \in \cC_l} \cC\pare{ \mathrm{SA}^{l+1}(\mX^l), \epsilon_{l+1}, \norm{\cdot}  }.
    \end{align*} 
    It is not hard to check the cardinality of $\cC_{l+1}$
    \begin{align*}
        |\cC_{l+1}| = \left|\bigcup_{\mX^l \in \cC_l} \cC\pare{ \mathrm{SA}^{l+1}(\mX^l), \epsilon_{l+1}, \norm{\cdot}  }\right| \leq |C_l| N_{l+1} \leq \prod_{l'=1}^{l+1} N_{l'}
    \end{align*}
    Let 
    \begin{align*}
        {\rho}_{l} &:= O\pare{ {(\alpha_1\alpha_2W^2(l) +  \alpha_1)^2  B^2(l) } \ln(2d^2)} \pare{K^2 + \frac{\alpha_1W^4(l)\pare{s_{l-1}\norm{\mX_* }  }^4}{d_K}  }   \\
        &\quad+ O\pare{ {\alpha_2^2W^2(l) B^2(l) (1 + \alpha_1^2 K^2W^2(l) )}   \ln (2dm)},  
    \end{align*}  
we have
    \begin{align*}
        \ln |\cC_{l+1}| \leq \sum_{l'=1}^{l+1} \ln N_{l'} \leq  \sum_{l'=1}^{l+1} \frac{\rho_{l'}}{\epsilon^2_{l'}}  {  s^2_{l'}\norm{\mX_{[N]}}^2}.
    \end{align*}
    Now it remains to verify $\cC_{L}$ is a cover of $\cH_{L}$.
    For any $\mX^{L} \in \cH_{L}$  we can find a $\hat{\mX}^{L} \in \cC_{L}$ such that
    \begin{align*}
        \norm{\mX^{L} - \hat{\mX}^{L}} &= \norm{\mathrm{SA}_{ {\mW}^{L}}(\mX^{L-1}) - \mathrm{SA}_{\hat{\mW}^{L }}(\hat{\mX}^{L-1}) } \\
        &\leq \norm{\mathrm{SA}_{ {\mW}^{L}}(\mX^{L-1}) - \mathrm{SA}_{ {\mW}^{L}}(\hat{\mX}^{L-1}) } +\norm{\mathrm{SA}_{ {\mW}^{L}}(\hat\mX^{L-1}) - \mathrm{SA}_{\hat{\mW}^{L}}(\hat{\mX}^{L-1}) } \\
        &\leq (\alpha_2 W^2(L)+1)\pare{  \alpha_1 KW(L)  + 1} \norm{  \mX^{L-1}   -   \hat\mX^{L-1} } + \epsilon_{L}\\
        & \leq \sum_{l =0}^{L}\prod_{j=l+1}^{L } (\alpha_2 W^2(j)+1)\pare{  \alpha_1 KW(j)  + 1} \epsilon_{l} 
    \end{align*}
    We choose $\epsilon_j = \pare{L\prod_{j=l+1}^{L } (\alpha_2 W^2(j)+1)\pare{  \alpha_1 KW(j)  + 1} }^{-1} \epsilon$, and let $s_{l+1 \mapsto L} := \prod_{j=l+1}^{L } (\alpha_2 W^2(j)+1)\pare{  \alpha_1 KW(j)  + 1}$.
Hence we have:
    \begin{align*}
        \frac{\rho_l}{\epsilon^2_l} = & \frac{\rho_l s_{l+1 \mapsto L}^2}{\epsilon^2}\norm{\mX_{[N]}}^2 \\ 
    \end{align*}
    
and conclude the covering number of $\cH_L$ as follows:
    \begin{align*}
        \ln \cN(\cH_L,\epsilon,\norm{\cdot}) = \ln |\cC_{L}| \leq  \sum_{l=1}^L  \frac{\rho_l}{\epsilon^2_l} \pare{s_{l}\norm{\mX_{[N]}}}^2 \\
        = \ln |\cC_{L}| \leq O\pare{ s^2_{L}\norm{\mX_{[N]}}^2 \sum_{l=1}^L \frac{\rho_l  }{\epsilon^2}}.
    \end{align*}
    Finally according to covering number fact (\ref{eq:covering fact}), 
    \begin{align}
         \ln \cN_{\infty}(\cG\circ\cH(S), \epsilon, \norm{\cdot} ) \leq \ln \cN(\cG\circ\cH(S), \epsilon, \norm{\cdot} ),
    \end{align}   
    we can conclude the proof.
    \end{proof}

 Now, equipped with covering number bound for the transformer, we are ready to show the generalization of MAE pre-training task.

\subsubsection{Proof of Lemma~\ref{lem:MAE gen}}\label{app:proof MAE main}

\begin{proof}
    Similar to the proof in CE section, we evoke Lemma~\ref{lem:local rademacher bound NN} with $c = O\pare{ s^2_{L+1}\norm{\tilde\mZ_{[N]}}^2 \sum_{l=1}^{L+1}  {\rho_l  } }$, where $s_{l},\rho_l$ are defined in Lemma~\ref{lem: transformer cover}.
    \begin{align*}
      \frakR_{\widehat\cS}\pare{\cL(r)}    
      &\leq 10\sqrt{\frac{Hr\cdot c}{N}} + 10\sqrt{\frac{ cHr}{N  }}\pare{\ln{\sqrt{br}} - \ln\pare{\frac{5}{2}\sqrt{\frac{Hr\cdot c}{N}}}}\\
      & = 10\sqrt{\frac{Hr\cdot c}{N}} + 10\sqrt{\frac{ cHr}{N  }}\ln{\pare{\frac{2}{5}\sqrt{\frac{bN}{Hc}}} }.
  \end{align*}
     We set $\phi(r) = 10\sqrt{\frac{Hr\cdot c}{N}} \cdot \max\left\{ 1,  \ln\pare{\frac{2}{5}\sqrt{\frac{bN}{Hc}}}\right\}$. Solving the follwoing equation to get $r^*$
    \begin{align*}
       \phi(r) &= 10\sqrt{\frac{Hr\cdot c}{N}} \cdot \max\left\{ 1,  \ln\pare{\frac{2}{5}\sqrt{\frac{bN}{Hc}}}\right\}  = r,\\
    \Longleftrightarrow r^* &= 100 {\frac{H\cdot c}{N}} \cdot \max\left\{ 1,  \ln\pare{\frac{2}{5}\sqrt{\frac{bN}{Hc}}}\right\}^2
    \end{align*}
    Now, according to Theorem~\ref{thm:local rademacher}, and the fact that 
    \begin{align*}
        A \leq B + C\sqrt{A} \Longrightarrow A \leq B+C^2 + \sqrt{B}C,
    \end{align*}
    we have
     \begin{align*}
        \cL_{\cU} (g\circ h) &\leq \cL_{\widehat \cU} (g\circ h) + 45r^* +  \pare{\sqrt{8r^* } + \sqrt{\frac{4b(\log(1/\nu)+6\log\log N)}{N}} }^2  \\
        &\quad + 20\frac{b(\nu+6\log\log N)}{N} \\
        & \quad+ \sqrt{\cL_{\widehat \cU} (g\circ h) + 45r^* +20\frac{b(\nu+6\log\log N)}{N} }\pare{\sqrt{8r^* } + \sqrt{\frac{4b(\log(1/\nu)+6\log\log N)}{N}} }
    \end{align*}
    Plugging $r^*$ and empirical risk minimizers $\hat g, \hat h$  will conclude the proof.
\end{proof}

\subsection{Proof of Theorem~\ref{thm: MAE main}}\label{app:proof MAE main}
\begin{proof}
Again, recall in Theorem~\ref{thm:main generalization}, the generalization bound of downstream task is given by
\begin{align*}
   \mathcal{E}_{\cT}(\hat{f}, \hat{h})    &\leq C_\beta \left(   \cE_{\cU}(\hat{g}, \hat{h}) \right)^\beta +4G_\phi {\frak{R}}_{\widehat \cT}( \cF\circ \hat{h})  + 4B_\phi\sqrt{\frac{\log(1/\nu)}{n}}  + 4B_\phi \norm{\cT-\cU_{\cX}}_{\mathrm{TV}} \\
   &\quad  +   \min_{f\in\cF}\cE_{\cT}(f,  h^*_{\cU} ).
\end{align*}
    Since in the previous subsection we prove the bounded transferrability and generalization of pre-training task, it remains to show the upper bound of representation-induced Rademacher complexity. 
    \begin{align*}
          {\frak{R}}_{\widehat \cT}(\phi\circ \mathcal{F}\circ \hat{h})  &= \mathbb{E}_{\boldsymbol \varepsilon \in \{\pm1\}^n} \left[\sup_{\btheta: \norm{\btheta}\leq R} \frac{1}{n}\sum\nolimits_{i=1}^{n} \varepsilon_i   \phi(\btheta^\top (\mathbf{1}^\top \hat{h} (  \mX_i ))^\top, y_i) \right] \\
          &\leq G_\phi \mathbb{E}_{\boldsymbol \varepsilon \in \{\pm1\}^n} \left[\sup_{\btheta: \norm{\btheta}\leq R} \frac{1}{n}\sum\nolimits_{i=1}^{n} \varepsilon_i    \btheta^\top (\mathbf{1}^\top \hat{h} (  \mX_i ))^\top  \right]\\
          &=  \frac{R G_\phi}{n} \mathbb{E}_{\boldsymbol \varepsilon  } \left\|\sum\nolimits_{i=1}^{n} \varepsilon_i (\mathbf{1}^\top \hat{h} (  \mX_i ))^\top  \right\|\\
           &\leq \frac{RG_\phi}{n} \sqrt{\mathbb{E}_{\boldsymbol \varepsilon  } \left\|\sum\nolimits_{i=1}^{n} \varepsilon_i  (\mathbf{1}^\top \hat{h} (  \mX_i ))^\top  \right\|^2}\\
           &\leq \frac{RG_\phi}{n} \sqrt{\sum\nolimits_{i=1}^{n} \left\| \mathbf{1}^\top \hat{h} (  \mX_i ) \right\|^2}\\
            &\leq \frac{RG_\phi}{n} \sqrt{\sum\nolimits_{i=1}^{n} K\left\|  \hat{h} (  \mX_i ) \right\|^2}\\
    \end{align*}
    where at first inequality we apply  Ledoux-Talagrand’s inequality to peel of Lipschitz loss $\phi(\cdot)$, and at last inequality we use the fact that $\varepsilon_i$ are i.i.d. with zero mean, so that the cross terms disappear. For each $\norm{\hat{h} (  \mX_i )}^2$, evoking (\ref{eq: norm of transformer output}) we have:
    \begin{align*}
       \norm{\hat{h} (  \mX_i )}    
     &\leq  \prod_{j=1}^l \pare{\alpha_2  W^2(j) +1} \pare{ W^2(j)\alpha_1  K   + 1  }\norm{\mX_i }  ,
    \end{align*}
    hence we arrive at 
    \begin{align*}
         {\frak{R}}_{\widehat \cT}(\phi\circ \mathcal{F}\circ \hat{h}) \leq \frac{R G_\phi \sqrt{\prod_{j=1}^l \pare{\alpha_2  W^2(j) +1}^2 \pare{ W^2(j)\alpha_1  K   + 1  }^2\sum_{i=1}^n\norm{\mX_i }^2 }}{n}.
    \end{align*}
    Plugging Lemmas~\ref{lem:MAE transfer} and~\ref{lem:MAE gen} as well as above bound will complete the proof.
\end{proof}

 \section{Proof of Convergence RadReg Algorithm}\label{app:proof convergence}

 In this section we provide the missing proofs from Section~\ref{sec:radreg}.  Then we provide the proof of convergence.

 \subsection{Convergence result of RadReg}
 In this section we provide formal version of convergence results for RadReg. First let us introduce the following Moreau envelope concept.
\begin{definition}[Moreau Envelope] A function $\Psi_{\rho} (\bw)$ is the $\rho$-Moreau envelope of a function $\Psi$ if $    \Psi_{\rho} (\bw) := \min_{\bw'\in \cW} \{ \Psi  (\bw') + \frac{1}{2\rho}\|\bw'-\bw\|^2\}$. 
 \end{definition}
We have the following property of the Moreau Envelope of a nonsmooth function:
\begin{lemma}\label{lemma: moreau}\citep{davis2019stochastic}
Let $\hat{\bw} = \arg\min_{\bw'\in \cW} \Psi  (\bw') + \frac{1}{2\rho}\|\bw'-\bw\|^2$, then we have the following facts:
$\|\hat{\bw} -  \bw\| \leq \rho\|\nabla \Phi_\rho(\bw)\|$, $\min_{\bg\in\partial\Psi(\hat{\bw})} \|\bg\| \leq \|\nabla \Phi_\rho(\bw)\|$.
\end{lemma}
Lemma~\ref{lemma: moreau} shows that, if we find a $\bw$ such that $\| \nabla \Psi_\rho(\bw)\|$ is small, then we can demonstrate that ${\bw}$ is near some point $\hat{\bm{x}}$ which is a near-stationary point of $\Psi$. We will use  $1/4L$-Moreau envelope of $\Psi$, following the setting in {~\citep{lin2020gradient,rafique2018weakly}}, and  state the convergence rate in terms of $\|\nabla \Psi_{1/4L} (\bw)\|$. We also define  quantity $\hat{\Delta}_{\Psi_{1/4L}} = \Psi_{1/4L}(\bw_0) - \min_{\bw\in\cW} \Psi_{1/4L}(\bw)$   that will be used in stating the convergence rates.
\begin{assumption}[Bounded Variance]~\label{assumption:Bounded Variance}
Let $\tilde \bz$ and $\tilde \bx$ be uniformly sampled from $\hat \cU$ and $\hat \cD$. Then, the variance of stochastic gradients is bounded:
\begin{align*}
    \E\norm{ \nabla \cL_{\hat \cU}(\bw;\tilde{\bz}) -  \nabla \cL_{\hat \cU}(\bw )  }^2 \leq \delta^2,\\
    \E\norm{ \nabla \frak{R}_j(\bv ,\bw ; \tilde \bx ) - \frac{1}{n}\sum\nolimits_{i=1}^n \nabla \frak{R}_j(\bv ,\bw; \bx_i ) }^2 \leq \delta^2.
\end{align*}
\end{assumption}

 \begin{assumption}[Smooth and Bounded Linear Head]~\label{assumption:smooth}
     $\cL_{\hat{\cU}}$ and $\frak{R}_j(\bv,\bw';\bx)$ are $L$ smooth w.r.t. $\bw$, $\forall j\in [B]$, and $\bx \in \cX$:
    \begin{align*}
        \norm{\nabla \cL_{\hat{\cU}}(\bw) - \nabla \cL_{\hat{\cU}}(\bw')} \leq L \norm{ \bw  -  \bw' },\\
        \norm{\nabla_{\bw} \frak{R}_j (\bv,\bw;\bx) - \nabla_{\bw} \frak{R}_j (\bv,\bw';\bx)} \leq L \norm{ \bw  -  \bw' }.
    \end{align*} 
     Also, we assume $\frak{R}_j(\bv,\bw;\bx)$ is linear in $\bv$, and $ \max_{\bv \in \cV} \norm{\bv}\leq D$.
 \end{assumption}

 \begin{assumption}[Lipschitzness]~\label{assumption:lip}
     $\cL_{\hat{\cU}}$ and $\frak{R}_j(\bv,\bw';\bx)$ are $G$ Lipschitz w.r.t. $\bw$, $\forall \bv \in \cV  , j\in [B]$, and $\bx \in \cX$, i.e.,
    \begin{align*}
        \norm{  \cL_{\hat{\cU}}(\bw) -   \cL_{\hat{\cU}}(\bw')} \leq G \norm{ \bw  -  \bw' },\\
        \norm{  \frak{R}_j (\bv,\bw;\bx) -   \frak{R}_j (\bv,\bw';\bx)} \leq G \norm{ \bw  -  \bw' }.
    \end{align*} 
  \end{assumption}

We are now ready to state the formal version of Theorem~\ref{thm:convergence of radreg informal} as follows.
\begin{theorem}[Convergence of {\sffamily{RadReg}} with Linear Top Layer] \label{thm:convergence of radreg}
Under Assumptions~\ref{assumption:Bounded Variance} and~\ref{assumption:smooth}, if we use {\sffamily{RadReg}} (Algorithm~\ref{algorithm: RPT_stoch} with one step update) to optimize (\ref{eq:radreg obj}), by choosing  {$\eta = \Theta\pare{\frac{\epsilon^6}{L^3 D^2 G}}$ and $\gamma = \Theta\pare{\frac{\epsilon^2}{L \delta^2}}$} it holds that: 
    \begin{align*}
      \frac{1}{T+1} \sum\nolimits_{t=0}^T \E \norm{\Psi_{1/4L} (\bw^t) }^2 \leq \epsilon^2,
    \end{align*}
    with the gradient complexity bounded by:
     \begin{align*}
     O\pare{\frac{BL^3(G^2+\frac{\delta^2}{n'}) D^2\frac{\delta^2}{n'}  \Delta_{\Psi_{1/4L}}  }{\epsilon^8}}.
 \end{align*} 
\end{theorem}
  We can see that the proposed optimization algorithm can find an $\epsilon$ stationary point with at most $O\pare{\frac{B}{\epsilon^8}}$ stochastic gradient evaluations. Since the complexity grows in terms of $B$, a proper sampling size of Rademacher variable is crucial.

In the rest of this section, we prove the convergence rate of {\sffamily{RadReg}}. The   proof idea mainly follows the framework developed in~\cite{lin2019gradient}. But before we state a few intermediate results that the main proof relies on. 
 
\begin{lemma}\label{lem: main recursion}
Under the conditions of Theorem~\ref{thm:convergence of radreg}, the following one iteration recursion relation holds true:
     \begin{align*}
     \eta \E\norm{\nabla \Psi(\bw^{t-1})}^2  
     &= \E [\Psi_{1/2L}(\bw^{t-1})] -\E [\Psi_{1/2L}(\bw^{t})]\\
     & \quad +4\eta L\frac{1}{B}\sum_{j=1}^B \E\left[ \frak{R}_j(\bv_j^*({\bw}^{t-1}), {\bw}^{t-1}) - \frak{R}_j(\bv_j^{t-1}, \bw^{t-1})   \right]  + \eta^2 L (G^2+ \frac{\delta^2}{n'}),
 \end{align*}
 where $\bv^*_j (\bw) := \arg\max_{\bv \in \cV} \frak{R}_j(\bv , {\bw}^{t-1}) $.
\end{lemma}
\begin{proof}
   Recall the definition of $\Psi$ and $\frak{R}_j$:
\begin{align} 
    & {\Psi}(\bw):=    \cL_{\hat{\cU}} (\bw) +  \lambda \frac{1}{B}\sum_{j=1}^B \left[\max_{\mV \in \cV}  \bv^\top \left(\frac{1}{n} \sum_{i}^n \sigma_i^j h_{\bw}(\bx_i) \right) \right],\\ 
     & \frak{R}_j (\bv,\bw):= \bv^\top \left(\frac{1}{n} \sum_{i}^n \sigma_i^j h_{\bw}(\bx_i). \right)
 \end{align}
 Also recall the definition of $\Psi$'s Moreau Envelope:
 \begin{align*}
     \Psi_{{1}/{4L}}(\bw) := \min_{\bw' \in \cW} \Psi(\bw') + 2L \norm{\bw - \bw'}^2.
 \end{align*}
 We define the proximal solution as:
\begin{align*}
    \hat{\bw}^{t} := \arg\min_{\bw' \in \cW}\Psi(\bw') + 2L \norm{\bw^t - \bw'}^2.
\end{align*}
With all aforementioned definitions  are in place, we proceed to proving the lemma. First, since $\hat \bw^{t-1}$ is not minimizer of  $\Psi(\cdot)+2L\norm{\cdot-{\bw}^{t}}^2$ we have
\begin{align*}
    \E[\Psi_{1/4L}(\bw^{t})] \leq \E[\Psi(\hat{\bw}^{t-1})] + 2L\norm{\hat{\bw}^{t-1} - \bw^t}^2
\end{align*}
Recall the updating rule for $\bw$ and $\bv_j$ as stated below:
\begin{align*}
    \bw^{t+1} &  = \bw^{t} - \eta \pare{\frac{1}{n'}\sum_{i=1}^{n'} \nabla \cL(\bw^{t};\tilde{\bx}_i^t) + \lambda \frac{1}{B}\sum_{j=1}^B \frac{1}{n'}\sum_{i=1}^{n'} \nabla_{\bw} \frak{R}_j(\bv^{t}_j,\bw^t;\tilde{\bx}_i^t)},\\
    \bv_j^{t+1} & = \bv_j^{t} + \gamma \lambda \frac{1}{B}\sum_{j=1}^B \frac{1}{n'}\sum_{i=1}^{n'} \nabla_{\bv} \frak{R}_j(\bv^{t}_j,\bw^t;\tilde{\bx}_i^t).
\end{align*}

Hence we can get the following relation by completing the square trick:
\begin{align*}
    \E \norm{\hat{\bw}^{t-1} - \bw^t}^2 &= \E \norm{\hat{\bw}^{t-1} - \bw^{t-1}}^2 \\
    &\quad + 2\eta \E\inprod{\hat{\bw}^{t-1} - \bw^{t-1}}{\nabla \cL(\bw^{t-1}) + \lambda \frac{1}{B}\sum_{j=1}^B \nabla_{\bv}\frak{R}_j(\bv_j^{t-1}, \bw^{t-1}) }    + \eta^2 (G^2+\frac{\delta^2}{n'}).
\end{align*}
According to $L$-smoothness of $\cL$ and  $\frak{R}_j$, we can re-write the inner product term as: 
\begin{align*}
   & \E\inprod{\hat{\bw}^{t-1} - \bw^{t-1}}{\nabla \cL(\bw^{t-1}) + \lambda \frac{1}{B}\sum_{j=1}^B \nabla_{\bv}\frak{R}_j(\bv_j^{t-1}, \bw^{t-1}) } \\
    &\leq \E\left[\cL(\hat{\bw}^{t-1}) - \cL(\bw^{t-1}) + \lambda \frac{1}{B}\sum_{j=1}^B  \pare{\frak{R}_j(\bv_j^{t-1}, \hat{\bw}^{t-1})- \frak{R}_j(\bv_j^{t-1}, \bw^{t-1}) } \right] + L \norm{\hat{\bw}^{t-1} - \bw^{t-1}}
\end{align*}
Notice the following fact about $\Psi$, $\Psi_{1/4L}$:
\begin{align*}
    \cL(\hat{\bw}^{t-1})  + \lambda \frac{1}{B}\sum_{j=1}^B  \frak{R}_j(\bv_j^{t-1}, \hat{\bw}^{t-1})  \leq \Psi(\hat{\bw}^{t-1}) \leq \Psi_{1/4L}( {\bw}^{t-1}) - 2L\norm{\hat{\bw}^{t-1}-{\bw}^{t-1}}^2 .
\end{align*}
The last inequality is because $\hat{\bw}^{t-1}$ is the minimizer of $\Psi(\cdot)+2L\norm{\cdot-{\bw}^{t-1}}^2 $.
As a result, the inner product is bounded by:
\begin{align*}
  \E\inprod{\hat{\bw}^{t-1} - \bw^{t-1}}{\nabla \cL(\bw^{t-1}) + \lambda \frac{1}{B}\sum_{j=1}^B \nabla_{\bv}\frak{R}_j(\bv_j^{t-1}, \bw^{t-1}) }   \leq \E\left[\Psi ( {\bw}^{t-1}) - F(\bv_j^{t-1}, \bw^{t-1})   \right] - L \E\norm{\hat{\bw}^{t-1} - \bw^{t-1}}.
\end{align*}
Finally, putting pieces together and using the fact that $\nabla \Psi_{1/4L} (\bw^{t-1}) = \norm{\hat{\bw}^{t-1} - \bw^{t-1}}/4L$ will conclude the proof:
 \begin{align*}
     \E [\Psi_{1/2L}(\bw^{t})] &\leq \E [\Psi_{1/2L}(\bw^{t-1})] +4\eta L\E\left[\Psi ( {\bw}^{t-1}) - F(\bv_j^{t-1}, \bw^{t-1})   \right] - 4\eta L \E\norm{\hat{\bw}^{t-1} - \bw^{t-1}} + 2\eta^2 L (G^2+\frac{\delta^2}{n'})\\
     &= \E [\Psi_{1/2L}(\bw^{t-1})] +4\eta L \frac{1}{B}\sum_{j=1}^B \E\left[ \frak{R}_j(\bv^*({\bw}^{t-1}), {\bw}^{t-1}) - \frak{R}_j(\bv_j^{t-1}, \bw^{t-1})   \right] \\
     & \quad - 4\eta L^2 \E\norm{\hat{\bw}^{t-1} - \bw^{t-1}} + 2\eta^2 L (G^2+\frac{\delta^2}{n'}).
 \end{align*}
\end{proof}

\begin{lemma}[Lemma D4 in~\citep{lin2019gradient}]\label{lem: R gap recursion} 
If $\frak{R}_j(\bv,\bw)$ is convex and smooth in $\bv$, $L$ smooth and $G$ Lipschitz in $\bw$, then under the dynamic of stochastic gradient descent ascent on $\bv$, we have the following statement holding: 
    \begin{align*}
  &\E [\frak{R}_j(\bv^*({\bw}^{t-1}), {\bw}^{t-1}) - \frak{R}_j(\bv_j^{t-1}, \bw^{t-1}) ]  \\
  & \leq \eta G \sqrt{G^2 + \delta^2/n' } (2t-2s-1) + \frac{1}{2\gamma} \pare{\E\norm{ \bv^*({\bw}^{s}) -  \bv_j^{t-1}  }^2 - \E\norm{ \bv^*({\bw}^{s}) -  \bv_j^{t}  }^2} \\
     &\quad + \E[\frak{R}_j(\bv^{t}, {\bw}^{t}) - \frak{R}_j(\bv^{t-1}, {\bw}^{t-1})] + \frac{\gamma \delta^2}{2n'}
 \end{align*} 

 and

 \begin{align*}
     \frac{1}{T+1}\sum_{t=0}^T \E [\frak{R}_j(\bv^*({\bw}^{t}), {\bw}^{t}) - \frak{R}_j(\bv_j^{t}, \bw^{t}) ] \leq \eta G S^2 \sqrt{G^2 + \sigma^2} + \frac{D^2}{2S \gamma} + \frac{\gamma \delta^2}{2n'} + \frac{\max_{\bv} \frak{R}(\bv,\bw^0) - \frak{R}(\bv^0,\bw^0)}{T+1}.
 \end{align*}
 where $D = \max_{\bv \in \cV} \norm{\bv} $,
\end{lemma}

\subsection{Proof of Theorem~\ref{thm:convergence of radreg}}
Now we are ready to present proof of Theorem~\ref{thm:convergence of radreg} by putting the above results together.
\begin{proof}
    Summing Lemma~\ref{lem: main recursion} from $t=0$ to $T$ yields:
         \begin{align*}
    \frac{1}{T+1} \sum_{t=0}^T  \E\norm{\nabla \Psi(\bw^{t})}^2  
     &= \frac{ \E [\Psi_{1/2L}(\bw^{0})] - \E [\Psi_{1/2L}(\bw^{T})]}{\eta(T+1)} + 4 L \pare{\eta G S\sqrt{G^2 + \sigma^2} + \frac{D^2}{2S \gamma} + \frac{\gamma \delta^2}{2}}\\
     &\quad + 4\frac{1}{B}\sum_{j=1}^B\frac{ L(\max_{\bv} \frak{R}_j(\bv,\bw^0) - \frak{R}_j(\bv^0,\bw^0))}{T+1}  +  4\eta  L \sqrt{G^2+\delta^2/n'}.
 \end{align*}
 Setting $S = \frac{D}{2} \sqrt{\frac{1}{\eta\gamma G \sqrt{G^2 + \delta^2/n'}}}$ yields:
        \begin{align*}
    \frac{1}{T+1} \sum_{t=0}^T  \E\norm{\nabla \Psi(\bw^{t})}^2  
     &= O\pare{\frac{\E [\Psi_{1/2L}(\bw^{0})] -\E [\Psi_{1/2L}(\bw^{T})]}{\eta(T+1)}} + O\pare{ L  D \sqrt{\frac{\eta G \sqrt{G^2 + \delta^2}}{ \gamma   }}   + \frac{L\gamma \delta^2}{2}}\\
     &\quad + O\pare{\frac{L(\max_{\bv} \frak{R}(\bv,\bw^0) - \frak{R}(\bv^0,\bw^0))}{T+1}  +  \eta  L \sqrt{G^2+\delta^2}}.
 \end{align*}
 Finally, by choosing $\eta = \Theta\pare{\frac{\epsilon^6}{L^3 D^2 G}}$ and $\gamma = \Theta\pare{\frac{\epsilon^2}{L \delta^2}}$, we can guarantee the stationary of past iterates:
 \begin{align*}
    \frac{1}{T+1} \sum_{t=0}^T  \E \norm{\Psi_{1/4L}(\bw^t)} \leq  {\epsilon} ,
 \end{align*}
with the gradient complexity bounded by
 \begin{align*}
     O\pare{\frac{BL^3(G^2+\delta^2/n') D^2\delta^2  \Delta_{\Psi_{1/4L}}  }{\epsilon^8}}.
 \end{align*}
 as stated.
 \end{proof}



%

 \section{Experiment Details}\label{sec:app:exp}

Recall we utilize the Masked AutoEncoder (MAE)~\citep{he2022masked} as the base unsupervised pre-training method. For models, we use the Tiny Vision Transform (TinyViT)~\citep{wu2022tinyvit} as the backbone for pre-training and use a $10$-way linear classifier on top of the encoder for fine-tuning.
The encoder $h$ sequentially contains one convolutional layer, 12 192-head attention blocks, and one layer-normalization layer.
The decoder $g$ for reconstructing images in MAE includes 4 192-head attention blocks followed by one linear layer.Details of hyperparameters for the experiments reported in Table~\ref{tab:mae} are included in~Table~\ref{tab:mae_hparams}. For {\sffamily{RadReg}}, we sample $\sigma$ for 50 times and solve the inner maximization by Adam optimizer with a learning rate of $0.001$ and a weight decay of $5\times 10^{-4}$.

\begin{table}[ht]
    \centering
    \small
    \begin{tabular}{lc}
    \toprule
    \textbf{Config} & \textbf{Value} \\
    \midrule
     Optimizer & AdamW \\ 
     Base learning rate & $1.5\times 10^{-4}$ \\
     Optimizer momentum & $\beta=0.9,0.95$ \\ 
     Batch size & 4096 \\
     Learning rate schedule & cosine decay \\
     Warmup epochs & 200 \\
     Augmentation & RandomResizedCrop \\
     Masking ratio & 75\% \\
     Pre-training epochs & 2000 \\
     Fine-tuning epochs & 300 \\
    \bottomrule
    \end{tabular}
    \caption{Pre-training setting. Fine-tuning follows the same setting except for the number of epochs.}
    \label{tab:mae_hparams}
\end{table}

\end{document}